\documentclass[twoside]{article}

\usepackage[accepted]{aistats2024}
\usepackage[round]{natbib}

\usepackage[utf8]{inputenc} % allow utf-8 input
\usepackage[T1]{fontenc}    % use 8-bit T1 fonts
\usepackage{hyperref}       % hyperlinks
\usepackage{url}            % simple URL typesetting
\usepackage{booktabs}       % professional-quality tables
\usepackage{amsfonts}       % blackboard math symbols
\usepackage{nicefrac}       % compact symbols for 1/2, etc.
\usepackage{microtype}      % microtypography
\usepackage{xcolor}         % colors

\usepackage{graphicx}
\usepackage{subcaption}
\usepackage{tikz}

\usepackage{amsmath}
\usepackage{amssymb}
\usepackage{mathtools}
\usepackage{amsthm}

\usepackage{algorithm}
\usepackage{algpseudocode}

\usepackage{framed}
\usepackage{wrapfig}

\theoremstyle{plain}
\newtheorem{theorem}{Theorem}[section]
\newtheorem{proposition}[theorem]{Proposition}
\newtheorem{lemma}[theorem]{Lemma}
\newtheorem{corollary}[theorem]{Corollary}
\theoremstyle{definition}
\newtheorem{definition}[theorem]{Definition}
\newtheorem{assumption}[theorem]{Assumption}
\theoremstyle{remark}

\def\minwrt[#1]{\underset{#1}{\text{minimize }}}
\def\argmin[#1]{\underset{#1}{\text{arg min }}}
\def\argmax[#1]{\underset{#1}{\text{arg max }}}

\begin{document}

\twocolumn[

\aistatstitle{Bures-Wasserstein Means of Graphs}

\aistatsauthor{ Isabel Haasler \And Pascal Frossard }

\aistatsaddress{ EPFL \And  EPFL } ]

\begin{abstract}
Finding the mean of sampled data is a fundamental task in machine learning and statistics.
However, in cases where the data samples are graph objects, defining a mean is an inherently difficult task.
We propose a novel framework for defining a graph mean via embeddings in the space of smooth graph signal distributions, where graph similarity can be measured using the Wasserstein metric.
By finding a mean in this embedding space, we can recover a mean graph that preserves structural information. 
We establish the existence and uniqueness of the novel graph mean, and provide an iterative algorithm for computing it. 
To highlight the potential of our framework as a valuable tool for practical applications in machine learning, it is evaluated on various tasks, including k-means clustering of structured aligned graphs, classification of functional brain networks, and semi-supervised node classification in multi-layer graphs.
Our experimental results demonstrate that our approach achieves consistent performance, outperforms existing baseline approaches, and improves the performance of state-of-the-art methods.
\end{abstract}

\section{INTRODUCTION}

Graphs, or networks, provide a compact representation for large data, describing for example biological, financial or social phenomena through interactions of elements in complex systems \citep{dong2019learning}.
Here, nodes represent entities, such as genes in a gene co-expression network \citep{sandhu2015graph}, areas in the brain \citep{bullmore2009complex}, or users in a social network \citep{majeed2020graph}, and edges describe correlations or interactions between them.
Due to their ability to model highly complex data in a wide range of applications, graphs have naturally become an increasingly important object of machine learning \citep{dong2020graph, hu2020open}. 

A key problem when considering graphs as statistical data objects is to define a mean for a set of graphs.
Various notions of graph means have been proposed, and they can be cast as versions of the following general optimization problem:
Given a set of graphs $G_1,\dots,G_m$ and a set of weights $\lambda_1,\dots,\lambda_m>0$ with $\sum_{j=0}^m \lambda_j = 1$, the weighted mean of the graphs is defined as
\begin{equation} \label{eq:graph_avg}
\bar{G} = \argmin[G \in \mathcal{G}] \sum_{j=1}^m \lambda_j d(G,G_j)^2, %\microcompress
\end{equation}
where $\mathcal{G}$ is a suitable space, or set, of graphs, and $d(\cdot,\cdot)$ is some dissimilarity measure for graphs, see, e.g., \citet{el2020orthonet, kang2020multi, kolaczyk2020averages, lunagomez2021modeling, mercado2019spectral,  meyer2022frechet, peyre2016gromov, xu2019gromov}. % xu2019scalable, mercado2019generalized, 
For instance, in case $(\mathcal{G},d)$ is a metric space, \eqref{eq:graph_avg} is the corresponding sample Fr\'echet mean \citep{dubey2019frechet}.
Thus, finding a meaningful graph mean relies on defining an appropriate distance 
$d(\cdot,\cdot)$ for graphs, which however, is a challenging
task in graph analytics and machine learning \citep{tantardini2019comparing, wills2020metrics}.

\begin{figure*}[]
\centering
\begin{minipage}{.35\textwidth}
\begin{framed}
\begin{center}
\textbf{Data}
\end{center}
\includegraphics[trim={0 0 0 0}, clip, width=\columnwidth]{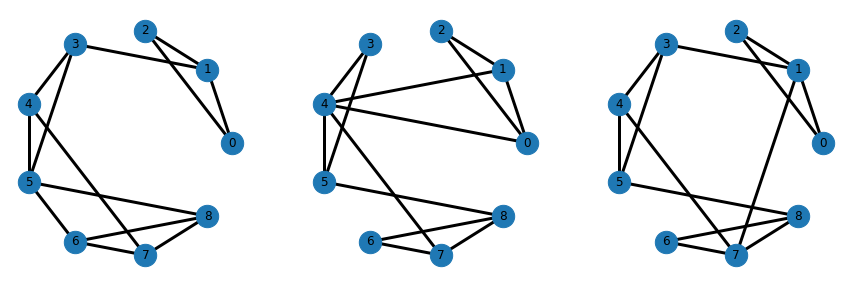}
\hrule
\begin{center}
\textbf{Means}
\end{center}
\begin{subfigure}[t]{.49\textwidth}
  \centering
  \includegraphics[trim={0 10 0 0}, clip,width=\textwidth]{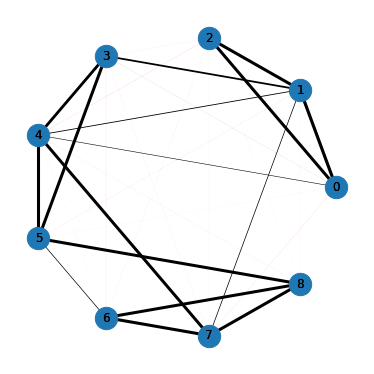}
  \caption*{Ours.}
  \end{subfigure} %\hspace{2pt}
  \begin{subfigure}[t]{.49\textwidth}
  \centering
\includegraphics[trim={0 10 0 0}, clip,width=\textwidth]{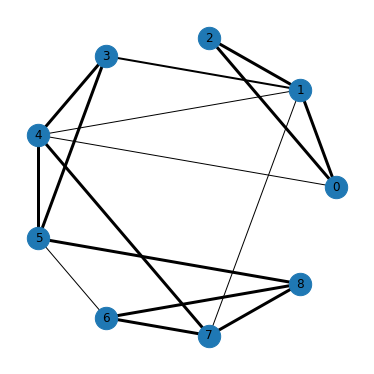}
  \caption*{Arithmetic.}
  \end{subfigure}
\begin{subfigure}[t]{.49\textwidth}
  \centering
\includegraphics[trim={0 10 0 0}, clip,width=\textwidth]{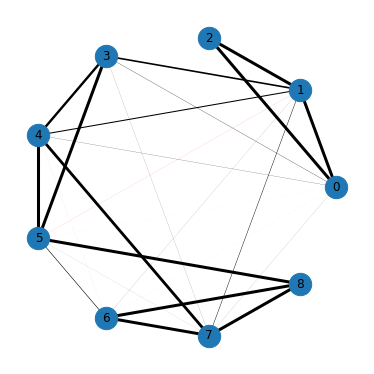}
  \caption*{Harmonic.}
\end{subfigure} %\hspace{5pt}
   \begin{subfigure}[t]{.49\textwidth}
  \centering
\includegraphics[trim={0 10 0 0}, clip,width=\textwidth]{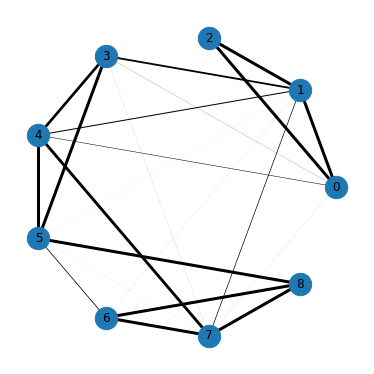}
  \caption*{Geometric.}
\end{subfigure}
\end{framed}
\end{minipage}
\hspace{1pt}
\begin{minipage}{.35\textwidth}
\begin{framed}
\begin{center}
\textbf{Data}
\end{center}
\includegraphics[trim={0 0 0 0}, clip, width=\textwidth]{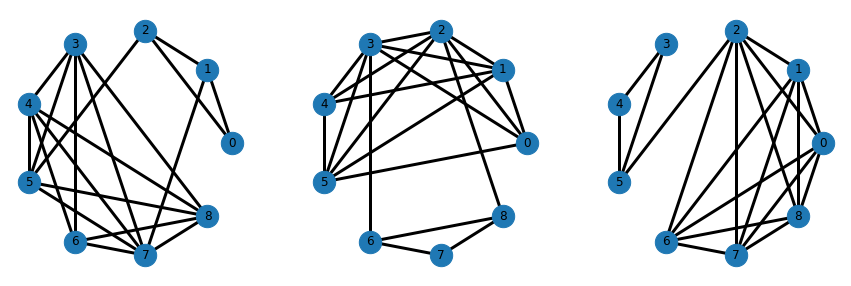}
\hrule
\begin{center}
\textbf{Means}
\end{center}
\begin{subfigure}[t]{.49\textwidth}
  \centering
\includegraphics[trim={0 10 0 0}, clip,width=\textwidth]{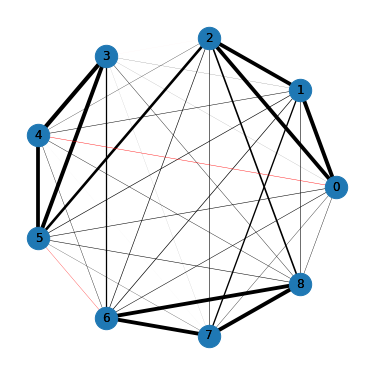}
  \caption*{Ours.}
    \end{subfigure} %\hspace{5pt}
  \begin{subfigure}[t]{.49\textwidth}
  \centering
\includegraphics[trim={0 10 0 0}, clip,width=\textwidth]{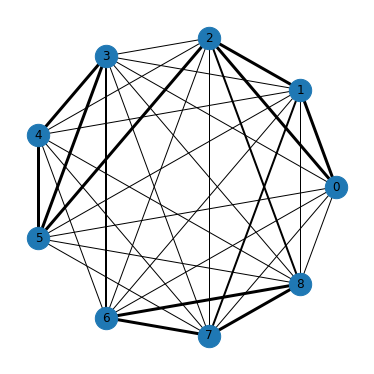}
  \caption*{Arithmetic.}
\end{subfigure} 
\begin{subfigure}[t]{.49\textwidth}
  \centering
\includegraphics[trim={0 10 0 0}, clip,width=\textwidth]{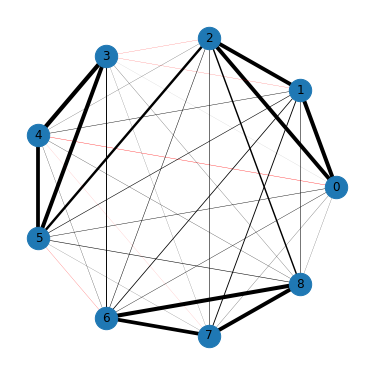}
  \caption*{Harmonic.}
\end{subfigure} %\hspace{5pt}
     \begin{subfigure}[t]{.49\textwidth}
  \centering
\includegraphics[trim={0 10 0 0}, clip,width=\textwidth]{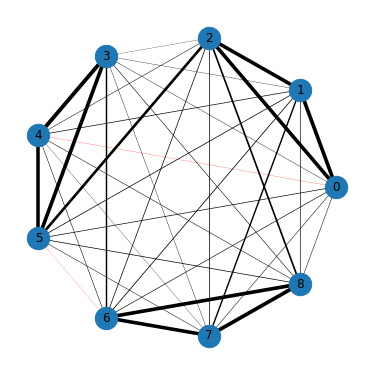}
  \caption*{Geometric.}
\end{subfigure}
\end{framed}
\end{minipage}%
\hfill %\hspace{5pt}
\begin{minipage}[b]{.26\textwidth}
\caption{Illustration of different means for two sets of graphs. The edge widths are proportional to the absolute value of edge weights, and red edges correspond to negative edge weights.
Our method is compared with the arithmetic mean, harmonic mean, and the geometric (Karcher) mean of the given graph Laplacians.
} \label{fig:example} 
\end{minipage}
\end{figure*}
In this work we propose a novel graph mean that utilizes a recently introduced optimal transport distance for graphs \citep{petric2019got, petric2020wasserstein} as the distance $d(\cdot,\cdot)$ in \eqref{eq:graph_avg}.
It turns out that using this graph optimal transport distance corresponds to moving the graph mean problem to an embedding space of
zero-mean normal distributions equipped with the Bures-Wasserstein distance. We therefore name our novel graph mean the \emph{Bures-Wasserstein mean} of graphs.
More precisely, in the embedding space a graph is represented by the probability distribution corresponding to a smooth graph signal that varies slowly over strongly connected nodes.
A key idea of our work is that by taking the Wasserstein mean in the embedding space we retrieve mean graphs that may preserve the structural information captured by the smooth graph signals of the input graphs.
We develop the theoretical foundation for the novel Bures-Wasserstein graph mean by proving the existence and uniqueness of a minimizer of the corresponding graph mean problem \eqref{eq:graph_avg}.
Moreover, we provide an iterative algorithm for computing the Bures-Wasserstein graph mean.
The potential of our proposed graph mean as a valuable tool for graph machine learning is illustrated through a range of numerical experiments, where it achieves consistent performance, outperforms existing baseline methods, and improves the performance of state-of-the-art methods. %\looseness=-1

The Bures-Wasserstein mean operates on a class of aligned and undirected graphs with edge weights that may be negative.
Such signed weighted graphs are a powerful tool to describe complex systems with positive as well as negative interactions. Therefore, they naturally appear in many applications. For instance, in gene regulatory networks and brain networks negative edge-weights can be used to describe inhibition or repression, and in social networks negative edge-weights often model antagonistic relationships \citep{chen2016characterizing, chen2016definiteness, olfati2007consensus}.

Overall, our main contributions are summarized as follows. 
\begin{itemize}  %\setlength\itemsep{-1pt}
\item We introduce a novel mean for a set of aligned graphs, which may have negative edge weights. It operates on the smooth graph signal distributions and can thus preserve structural information of the input graphs.  
\item The proposed Bures-Wasserstein mean is defined as the solution to an optimization problem of the form \eqref{eq:graph_avg}, and
we prove the existence and uniqueness of a solution to this problem. 
\item We present a new iterative algorithm for computing the Bures-Wasserstein graph mean.
\item We illustrate the potential of our framework for graph machine learning problems, like k-means clustering of graphs and functional brain network classification. Moreover, we demonstrate that the performance of a state-of-the-art method for semi-supervised node classification in multi-layer graphs can be improved with the Bures-Wasserstein mean.
\end{itemize}

Since the Bures-Wasserstein mean acts on signed weighted graphs that are aligned, it is
most related to a class of means that act on the graphs' Laplacian matrices, for example arithmetic, harmonic, or various geometric and power means \citep{el2020orthonet, mercado2019generalized, mercado2019spectral}.
Figure~\ref{fig:example} illustrates the advantage of our suggested mean over existing notions of means for graph Laplacian matrices.
Our Bures-Wasserstein mean is able to identify communities and captures connections between them in both settings.
Other methods do not perform well for both tasks. In particular, the arithmetic mean $\sum_{j=1}^m \lambda_j L_j$ of a set of graph Laplacians $L_1,\dots,L_m$ averages the weight of each edge individually. This local averaging does often not capture global properties of the graphs, such as communities (Figure~\ref{fig:example} right).
On the other hand, the harmonic mean $ \left( \sum_{j=1}^m \lambda_j L_j^{\dag} \right)^{\dag}$ is based on the arithmetic mean of the pseudo-inverses of the graph Laplacians. The pseudo-inverses describe correlations between all nodes on a global level, but do not describe local graph properties well. In the case of sparse input graphs, the harmonic mean is thus prone to creating spurious edges that are not present in any of the original graphs (Figure~\ref{fig:example} left). The geometric mean \citep{lim2012matrix} suffers from the same issue.

Other graph means focus on the case of unweighted graphs.
For instance, the graph mean with respect to the Hamming distance is in fact a modification of the arithmetic mean of Laplacians for unweighted graphs \citep{meyer2022frechet}.
We finally note that another optimal transport based graph mean has been defined using the Gromov-Wasserstein distance \citep{peyre2016gromov, vayer2019optimal, xu2019scalable, xu2019gromov},
which is however fundamentally different from the Bures-Wasserstein mean. 
In particular, in contrast to Gromov-Wasserstein barycenters, our method considers edge weights and preserves node correspondences, and is thus suitable for
multi-layer graphs or graphs with predefined node sets. \looseness -1

To the best of our knowledge, the novel Bures-Wasserstein graph mean is the first graph mean that utilizes the relation between graphs and their smooth graph signal representations,
which is a powerful connection, that is well established for graph learning settings \citep{dong2016learning, dong2019learning, dong2020graph, kalofolias2016learn}.
We expect that these theoretical and practical results provide a solid foundation for the use of the Bures-Wasserstein graph mean as a novel powerful operator in the graph machine learning toolbox.

%\microcompress

\section{BACKGROUND}

In this section we present some background on graph signal processing and optimal transport theory that will be used in Section~\ref{sec:main} to prove our main results.

\subsection{Graph signal processing} \label{subsec:gsp}
Graph signal processing aims to generalize classical signal processing concepts, such as Fourier transforms and filters to the graph domain \citep{dong2020graph, mateos2019connecting, ortega2018graph}.
Let $G=(V,E, W)$ be an undirected weighted graph, where $V$ denotes a set of $N$ vertices, $E$ denotes a set of edges, and $W\in  \mathbb{R}^{N\times N}$ is the weighted adjacency matrix with element $W_{ik}$ denoting the weight\footnote{Note that we also allow for negative edge weights in contrast to the previous works \citep{ petric2020wasserstein, petric2019got}.} of edge $(i,k)$. 
Moreover, define the degree matrix $D \in \mathbb{R}^{N\times N}$, which is a diagonal matrix with elements $D_{ii} = \sum_{k=1}^n W_{ik}$.
Then, the signed
weighted graph Laplacian matrix is defined as $L=D-W$.

From a signal processing perspective, features on a graph can be modeled as a graph signal, which is a mapping $x:V \to \mathbb{R}^N$.
Consider the spectral decomposition of the corresponding graph Laplacian $L=V \Lambda V^T$, where $\Lambda \in \mathbb{R}^{N\times N}$ is a diagonal matrix containing the eigenvalues of $L$, and the columns of $V\in \mathbb{R}^{N\times N}$ are the corresponding eigenvectors.
Then a graph Fourier transform of a graph signal $x$ can be defined as $\tilde x = V^T x$.
The graph Fourier transform provides a decomposition of the graph signal into different frequency components.
To see this consider the total variation of a graph signal, defined as
\begin{equation} \label{eq:tv}
    TV(x) = x^T L x = \sum_{i,j=1}^N W_{ij} (x_i-x_j)^2.
\end{equation}
By plugging in an eigenvector $v_k$ to the corresponding eigenvalue $\lambda_k$, one can see that $TV(v_k) = v_k^T L v_k = \lambda_k$.
Hence, small eigenvalues are associated with small variation over connected nodes, which can be interpreted as low frequency contents of the signal $x$.

Moreover, the spectral decomposition of the graph Laplacian gives rise to a factor analysis model \citep{tipping1999probabilistic} for graph signals, see \citet{dong2016learning}. Namely, assume that the (normalized) observed graph signal $x$ is generated by $x= V h + \epsilon$, where $h\in \mathbb{R}^N$ represent latent variables in Fourier space, and $\epsilon \sim \mathcal{N}(0,\sigma_\epsilon^2 I_n)$ for a scalar $\sigma_\epsilon$.
By imposing the Gaussian prior $h\sim \mathcal{N}(0,\Lambda^\dag)$ on the latent variables, one assumes that the energy of the graph signal lies mostly in the low frequency components.
Using this prior in the factor analysis model generating $x$, the resulting noise-less signal has the distribution $x \sim \mathcal{N}(0,L^\dag)$.
That is, signals sampled from this distribution are smooth in the sense that they have small total variation \eqref{eq:tv}, i.e., they vary slowly over strongly connected nodes. \looseness=-1

\subsection{Optimal transport} \label{subsec:ot}
Optimal transport is a classical topic in mathematics that has originally been used in economics and operations research \citep{villani2021topics}, and has now become a popular tool in fields such as machine learning \citep{arjovsky2017wasserstein, peyre2019computational}, computer vision \citep{solomon2016entropic}, and signal processing \citep{elvander2020multi, kolouri2017optimal}.
The optimal transport problem seeks a map that moves the mass from one probability distribution to another one in the most efficient way with respect to an underlying cost, which is often defined as the squared Euclidean distance. 
Let $\mu_0$ and  $\mu_1$ be two probability measures\footnote{Note that a probability measure defines a probability distribution. We often use the terms measure and distribution interchangeably.}
on the space $\mathbb{R}^n$.
We denote the set of probability measures in $\mathbb{R}^n \times \mathbb{R}^n$ with marginals $\mu_0$ and $\mu_1$ as $\Pi (\mu_0,\mu_1)$.
The measures $\pi \in \Pi(\mu_0,\mu_1)$ are called transport plans between the given marginals, and $\pi(x,y)$ is associate with the infinitesimal amount of mass transported from $x$ to $y$. 
The (Kantorovich) optimal transport problem is to find the transport plan between $\mu_0$ and $\mu_1$ with the smallest associated transportation cost, that is to minimize
\begin{equation} \label{eq:W2squared} 
\left(\mathcal{W}_2(\mu_0,\mu_1) \right)^2 = \inf_{\pi \in \Pi(\mu_0,\mu_1)} \int_{\mathbb{R}^n \times \mathbb{R}^n} \|x-y\|_2^2 \ d\pi(x,y).
\end{equation}
The optimal objective value of this problem defines the Wasserstein distance $\mathcal{W}_2(\cdot,\cdot)$ \citep{villani2021topics}.
This distance defines a metric on the space of absolutely continuous probability measures on $\mathbb{R}^n$ with finite second moment, denoted $\mathcal{P}_{2}(\mathbb{R}^n)$,
and the Fr\'echet mean on the metric space $(\mathcal{P}_{2}(\mathbb{R}^n), \mathcal{W}_2)$ is called the Wasserstein barycenter \citep{agueh2011barycenters}.
However, note that the Wasserstein barycenter problem can even be defined for measures that are not necessarily absolutely continuous.
Given a set of measures $\mu_1,\dots,\mu_m$ 
and a set of weights $\lambda_1,\dots,\lambda_m > 0$ such that $\sum_{j=1}^n \lambda_j = 1 $, the weighted Wasserstein barycenter is the minimizer of the optimization problem
\begin{equation} \label{eq:barycenter} 
\bar \mu = \argmin[\mu] \sum_{j=1}^m \lambda_j \left( \mathcal{W}_2(\mu,\mu_j) \right)^2.
\end{equation}

Although the optimal transport problem \eqref{eq:W2squared} is in general challenging to solve, in case the given probability distributions are Gaussian it has an analytical solution \citep{ dowson1982frechet, olkin1982distance, takatsu2010wasserstein}. More precisely, consider two measures $\mu_0 \sim \mathcal{N}(0,\Sigma_0)$ and $\mu_1 \sim \mathcal{N}(0,\Sigma_1)$, describing zero-mean 
normal distributions with covariance matrices $\Sigma_0, \Sigma_1$.
Then the Wasserstein distance between these probability distributions is given by 
\begin{equation} \label{eq:BWdistance}
\begin{aligned}
 \left(\mathcal{W}_2(\mu_0,\mu_1) \right)^2 = & \ \text{trace} \left(\Sigma_0\right) + \text{trace} \left(\Sigma_1\right) \\
&  - 2 \cdot \text{trace}\left(  \sqrt{\Sigma_0^{1/2} \Sigma_1 \Sigma_0^{1/2} } \right) . 
\end{aligned}
\end{equation}
Unfortunately, the barycenter \eqref{eq:barycenter} of a set of normal distributions given by $\mu_{j} \sim \mathcal{N}(0, \Sigma_j)$, for $j=1,\dots,m$, does not have an analytical solution when $m>2$ \citep{agueh2011barycenters}.
However, 
in the case that the given distributions are non-degenerate 
it is known that the barycenter is also a zero mean normal distribution $\bar{\mu} \sim \mathcal{N}(0, \bar{\Sigma})$, and its covariance matrix $\bar{\Sigma}$ is the unique positive definite solution to the matrix equation \citep{agueh2011barycenters} 
\begin{equation} \label{eq:bary_fb}
S = \sum_{j=1}^m \lambda_j \left( S^{1/2} \Sigma_j S^{1/2} \right)^{1/2}. 
\end{equation}
Note that zero-mean Gaussians are fully parametrized by their covariance matrices. In fact, there are important connections between the space of normal distributions, and the space of symmetric positive definite matrices.
Namely, the Wasserstein distance also defines a metric on the space of symmetric positive-definite matrices, and is in this context also called the Bures-Wasserstein metric \citep{bhatia2019bures}.
The barycenter problem \eqref{eq:barycenter} can thus be interpreted as taking the mean of a set of symmetric positive-definite matrices $\Sigma_0,\dots,\Sigma_m$ with respect to the Bures-Wasserstein metric \citep{bhatia2019inequalities}, and this problem is also solved by the solution to the matrix equation \eqref{eq:bary_fb}.
In the next Section we will extend these results in order to define a novel mean for a set of graphs.

\section{BURES-WASSERSTEIN GRAPH MEANS} \label{sec:main}

The core idea of the present work is that embedding graphs into a space of distributions equipped with the distance \eqref{eq:BWdistance} provides an elegant and simple framework for considering graphs as statistical data.
Specifically, in this section we define a notion of graph mean \eqref{eq:graph_avg} through an instance of the Wasserstein barycenter problem \eqref{eq:barycenter} and characterize it as the solution of a matrix equation similar to \eqref{eq:bary_fb}. Based on this, we provide an algorithm for computing the graph mean.
\subsection{Graph mean}
In this work we will use the following rather general assumption.
\begin{assumption} \label{ass:psd}
Let $L\in\mathbb{R}^{N \times N}$ be a signed
weighted graph Laplacian that is positive semi-definite and has only one zero eigenvalue.
\end{assumption} %%\microcompress
This assumption is discussed in detail in the supplementary material.
Here we note that all connected graphs with non-negative edge weights have Laplacians that satisfy Assumption~\ref{ass:psd}, and conversely all signed
weighted graphs with Laplacian matrices satisfying Assumption~\ref{ass:psd} are connected \citep{chen2016characterizing, chen2016definiteness}.

Graphs with semi-positive Laplacian matrix can further be represented by their smooth graph signal distribution, as described in Section~\ref{subsec:gsp}.
Based on the connection between graphs and their smooth signal representation, we define the Bures-Wasserstein distance for graphs as follows. 
\begin{definition} \label{def:dist}
Let $G_0$ and $G_1$ be two aligned graphs with signed weighted graph Laplacian matrices $L_0$ and $L_1$ satisfying Assumption~\ref{ass:psd}, and consider two measures $\mu_{G_j} \sim \mathcal{N}(0, L_j^\dag)$, for $j=0,1$.
The Bures-Wasserstein distance between the graphs $G_0$ and $G_1$ is defined as $d_{BW}(G_0,G_1)  := \mathcal{W}_2(\mu_{G_0}, \mu_{G_1})$.
\end{definition} 
Let $\mathcal{G}$ be a set of aligned signed weighted graphs with $N$ nodes and Laplacians that satisfy Assumption~\ref{ass:psd}. We define the Bures-Wasserstein mean of the graphs $G_1,\dots,G_m \in \mathcal{G}$ analogously to \eqref{eq:graph_avg} as the solution to
\begin{equation} \label{eq:graph_avg_BW}
\begin{aligned}
\bar{G} &= \argmin[G \in \mathcal{G}] \sum_{j=1}^m \lambda_j d_{BW}(G,G_j)^2 \\ %
&= \argmin[G \in \mathcal{G}] \sum_{j=1}^m \lambda_j \mathcal{W}_2(\mu_{G}, \mu_{G_j})^2.  
\end{aligned}
\end{equation}
We now characterize the solutions to the optimization problem \eqref{eq:graph_avg_BW}.
Note that the Bures-Wasserstein graph mean has similarity with the Wasserstein barycenter \eqref{eq:barycenter}. However, it is important to note that the theory from Section~\ref{subsec:ot} cannot be applied directly.
Firstly, the feasibility set in \eqref{eq:graph_avg_BW} is more restrictive than in \eqref{eq:barycenter}, since the barycenter distribution $\mu_G$ must be of the specific form $\mu_G \sim \mathcal{N}(0, L^\dag)$, where $L$ satisfies Assumption~\ref{ass:psd}.
Secondly, smooth graph signal distributions are degenerate Gaussians, but the characterization \eqref{eq:bary_fb} and the uniqueness of the barycenter have only been proved for non-degenerate Gaussians \citep{agueh2011barycenters, xia2014synthesizing}.
Thus, extending the theory in Section~\ref{subsec:ot} we now prove the existence and uniqueness of a solution $\bar{G}$ to problem \eqref{eq:graph_avg_BW}.

\begin{theorem} \label{thm:bary_unique}
Let $G_j$, for $j=1,\dots,m$, be %connected
graphs with signed weighted graph Laplacians $L_j$ that satisfy Assumption~\ref{ass:psd}. 
Then the Bures-Wasserstein mean \eqref{eq:graph_avg_BW} of these graphs is described by the Laplacian matrix $L=S^\dag$, where $S$ is the unique positive semi-definite symmetric solution to 
\begin{equation} \label{eq:bary_eq_degen}
S = \sum_{j=1}^m \lambda_j \left( S^{1/2} L_j^\dag S^{1/2} \right)^{1/2} 
\end{equation}
that satisfies $\text{range}(S)=\mathcal{R} = (\text{span}\{\mathbf{1}_N\})^\perp$. 
\end{theorem}
\begin{proof}[Proof sketch:]
We first prove the existence and uniqueness of a solution to the Wasserstein barycenter problem \eqref{eq:barycenter} in the special case, where the given distributions are degenerate Gaussians, which are embeddings of graphs in $\mathcal{G}$.
We then show that this solution in fact describes a smooth graph signal.
An essential observation here is that Assumption~\ref{ass:psd} guarantees that the degeneracy of all given measures $\mu_j \sim \mathcal{N}(0, L_j^\dag)$ lies in the subspace $\mathcal{R} = (\text{span}\{\mathbf{1}_N\})^\perp \subset \mathbb{R}^N$.

See the supplementary material for a full proof.
\end{proof}

\subsection{Algorithm}
We now introduce a computational method for finding the Bures-Wasserstein graph mean, see Algorithm~\ref{alg:bary}.
This method is based on a fixed point iteration for solving the matrix equation \eqref{eq:bary_fb} that was first introduced in \citet[Theorem~4.2]{alvarez2016fixed}.
This fixed point iteration can be directly applied to find a solution to the matrix equation \eqref{eq:bary_eq_degen},  which characterizes 
the Bures-Wasserstein graph mean.
In addition, we exploit spectral properties of graph Laplacians as described in the following.
\begin{proposition} \label{prop:pd_transform}
Let $L_1,\dots,L_m \in \mathbb{R}^{N\times N}$ satisfy Assumption~\ref{ass:psd}, and let $\mathbf{1}_{N \times N}\in\mathbb{R}^{N \times N}$ denote a matrix of ones.
Then the matrices $\Sigma_j = L_j^\dag + \frac{1}{N} \mathbf{1}_{N \times N}$, for $j=1,\dots,m$ are symmetric and strictly positive definite.
Moreover, the Bures-Wasserstein mean \eqref{eq:graph_avg_BW} of the graphs with Laplacians $L_1,\dots,L_m$ is the graph with Laplacian $L = \bar{\Sigma} - \frac{1}{N} \mathbf{1}_{N \times N}$, where $\bar{\Sigma}$ is the Bures-Wasserstein barycenter of $\Sigma_1,\dots,\Sigma_m$.
\end{proposition}
\begin{proof}
    See supplementary material.
\end{proof}
Proposition~\ref{prop:pd_transform} permits to transform the singular graph Laplacians into positive definite matrices. 
This makes the algorithm efficient and stable, as we do not need to take square roots and pseudo-inverses of singular matrices.
\begin{algorithm}[tb] \caption{Bures-Wasserstein mean of graphs.} \label{alg:bary}
\begin{algorithmic}
\State {\bf Given:} $L_1,\dots,L_m \in \mathbb{R}^{N\times N}$ satisfying Assumption~\ref{ass:psd}, and initial SPD matrix $S\in \mathbb{R}^{N\times N}$
\State $\displaystyle{ \Sigma_j \gets \left(L_j  + \frac{1}{N} \mathbf{1}_{N\times N} \right)^{-1}}$ for $j=1,\dots,m$
\While{Not converged}
\State $ \displaystyle{ S \gets S^{-1/2} \left( \sum_{j=1}^m \lambda_j ( S^{1/2} \Sigma_j S^{1/2} )^{1/2} \right)^2 S^{-1/2},}$ 
\EndWhile
\State {\bf Return:} $\displaystyle{L \gets S^{-1} - \frac{1}{N} \mathbf{1}_{N\times N} }$
\end{algorithmic}
\end{algorithm} 
The fixed-point iteration in Algorithm~\ref{alg:bary} inherits the convergence guarantee proved in \citet{alvarez2016fixed}, and in practice we typically observe convergence within a few iterations.

\begin{table*}[]
\setlength{\tabcolsep}{4pt}
\centering
\caption{Difference of the perturbed graphs' mean to the original graph with respect to several metrics. } \label{tab:graph_fusion}
\begin{tabular}{c  c c c c c c}
& & & & & & Participation \\
 & B-W distance & Laplacian  & Covariance & Degree centrality & Modularity &  coefficient \\
 \midrule 
B-W mean & $ \mathbf{0.097 \pm 0.043}$ & $ \mathbf{2.26 \pm 0.11}$ & $\mathbf{0.15 \pm 0.12}$ & $\mathbf{0.036 \pm 0.002}$ & $ \underline{\mathbf{0.011 \pm 0.003}}$ &  $\underline{\mathbf{0.25 \pm 0.09}}$ \\
Arithmetic  & $ 0.170 \pm 0.073$ & $ \underline{\mathbf{1.85 \pm 0.10}} $ & $0.26 \pm 0.18$ & $ \underline{\mathbf{0.026 \pm 0.003}}$ & $0.015 \pm 0.003$ & $0.37 \pm 0.13$ \\
Harmonic & $\underline{\mathbf{0.095 \pm 0.037}}$ & $2.87 \pm 0.19$ & $\underline{\mathbf{0.13 \pm 0.10}}$ & $0.047 \pm 0.004$ & $\mathbf{0.012 \pm 0.002}$ & $\mathbf{0.26 \pm 0.09}$ \\
Geometric & $0.286 \pm 0.067$ & $6.42 \pm 0.31$ & $0.34 \pm 0.14$ & $0.091 \pm 0.009$ & $0.015 \pm 0.008$ & $0.40 \pm 0.13 $ \\
\end{tabular}
\end{table*}

\subsection{Extensions}
A special case of the Bures-Wasserstein graph mean problem is the setting where only two graphs are given, i.e., where $m=2$ in \eqref{eq:graph_avg_BW}.
Namely, by defining the weights as $\lambda_1=t$ and $\lambda_2=1-t$ for $t\in (0,1)$, the solution to \eqref{eq:graph_avg_BW} can be understood as the interpolation between two graphs.
For the classical Wasserstein problem this is also called displacement interpolation, and the solution for $t\in (0,1)$ defines a geodesic in the space $(\mathcal{P}_{2}(\mathbb{R}^n), \mathcal{W}_2)$ \citep{mccann1997convexity, villani2021topics}.
If the two given measures are Gaussian, there is a closed-form solution to the displacement interpolation problem.
These results can be extended to problem \eqref{eq:graph_avg_BW}, which gives a closed-form solution of the graph Laplacian interpolation. % that can be interpreted as the trajectory of a dynamic graph between two time instances.
For completeness, we present the following result for this setting.
\begin{theorem} \label{thm:geodesic}
Let $G_0$ and $G_1$ be two graphs with signed weighted graph Laplacians $L_0$ and $L_1$ that satisfy Assumption~\ref{ass:psd}.
Then the Bures-Wasserstein mean of these two graphs with weights $1-t$ and $t$ for $t\in(0,1)$ has the signed weighted graph Laplacian $L_t = S_t^\dag$, where 
\begin{equation} \label{eq:interpolation}
S_t = L_0^{1/2} \! \left(  (1-t) L_0^\dag +t \left(  L_0^{\dag/2} L_1^\dag L_0^{\dag/2}  \right)^{1/2} \right)^2 \! L_0^{1/2} \!. %
\end{equation}
\end{theorem}
\begin{proof}
    See supplementary material.
\end{proof}
It is worth noting the connection between the analytical solution in the interpolation setting from Theorem~\ref{thm:geodesic} and the iteration in Algorithm~\ref{alg:bary}.
In particular, in the case that $m=2$, and when choosing the starting point $S_0=L_1^\dag$, the iteration in Algorithm~\ref{alg:bary} corresponds to \eqref{eq:interpolation} and thus converges in one step. 

Finally, we note that an extension of the Bures-Wasserstein distance for graphs can take into account filters on the graph signal distributions. This permits to emphasize different types of spectral information in the graph comparison \citep{petric2021fgot}.
For instance, employed with a low-pass filter, the Bures-Wasserstein distance predominantly captures differences in the global graph structures, whereas a high-pass filter focuses on local structures.
These distances can also be utilized in the graph mean problem \eqref{eq:graph_avg_BW} when the graph filter is bijective for the set of Laplacian matrices satisfying Assumption~\ref{ass:psd}. For more details see the supplementary material. 

\section{EXPERIMENTS} \label{sec:exp}

In this Section we illustrate the behavior of the Bures-Wasserstein graph mean experimentally in several machine learning problems.
In particular, we show that in many settings it provides a better representation of a set of aligned graphs than other graph means.
In all experiments we set the weights in \eqref{eq:graph_avg_BW} to $\lambda_j=1/m$ for $j=1,\dots,m$, where $m$ denotes the number of graphs.

\subsection{Graph fusion} \label{subsec:fusion}
We first illustrate that the Bures-Wasserstein mean graph preserves graph characteristics of the input graphs well. 
We create a stochastic block model graph with $50$ nodes and $5$ communities. The edge probability within the clusters is $0.3$, and the edge probability between clusters is $0.1$. From this graph we generate $100$ graphs by removing and adding $10$ random edges, respectively, while enforcing that all graphs are connected.
The Bures-Wasserstein barycenter of these $100$ graphs is computed using Algorithm~\ref{alg:bary} and compared to the original graph in terms of several metrics.
We also compute the same metrics for the arithmetic, harmonic, and geometric mean of the graph Laplacians.
This experiment is repeated $1000$ times, and the mean error with respect to the different graph metrics are presented in Table~\ref{tab:graph_fusion}.
Here, the error in the Laplacian and covariance matrix is measured in Frobenius-norm.
The degree centrality (for weighted graphs also called strength), modularity, and participation coefficient are computed as defined in \cite{oehlers2021graph}.
The difference in degree centrality and participation coefficient are presented in mean square error over all nodes, and the modularity difference in absolute value.

We see that our proposed method is among the two best-performing methods for all tested graph metrics.
Degree centrality describes local properties of the graph, and is thus represented well by the arithmetic mean, which takes the average weight of each edge.
Modularity on the other hand describes the global graph structure, and this is well represented by the harmonic mean.
The participation coefficient typically captures both local and global phenomena.
We see that the Bures-Wasserstein mean preserves all tested graph metrics well, which shows that it successfully fuses both local and global structural information from the input graphs. 

We note here that the geometric mean is not preserving any of the graph metrics well.
As the methods based on the geometric mean did not perform as well as the other baseline methods in our conducted experiments, we chose to omit them in the following sections for clarity and brevity.

\subsection{K-means clustering of graphs} \label{subsec:kmeans}
Next we apply our proposed framework to an unsupervised learning setting. 
The task in this experiment is to identify graphs that have the same number of communities using k-means clustering.
The dataset is constructed as follows.
We generate connected graphs with $50$ nodes and a varying number of communities, ranging from $1$ up to $N_C$, where $N_C \in \{4,5,6\}$.
These make up the $N_C$ classes that we aim to identify, and each class contains $20$ graphs.
The probability for two nodes in a community to be connected is $p$, where $p\in \left\{ 0.2, 0.25, 0.3 \right\}$.
The communities are connected on a line, that is, each community is connected to at most two other communities, and the connectivity between the neighboring communities is randomly picked from two possible patterns.
Each of these two patterns constructs one edge between the communities, between a different set of nodes\footnote{Example graphs of this dataset can be found in the supplementary material.}.
For each class of graphs, the edges within clusters thus have high variability, while the edges between clusters have low variability.
Successfully clustering these highly structured types of graphs requires the ability to capture both the community structure as well as the inter-community structure.

\begin{table*}
\setlength{\tabcolsep}{4pt}
\centering
\caption{Functional brain network misclassification rate ($\%$). } \label{tab:brain_classification} 
\begin{tabular}{l  c c c c c}
& & & & & Topological \\
 & Bures-Wasserstein & Arithmetic  & Harmonic  &  Topological &  + arithmetic \\
 \midrule 
Motor execution: left vs. right hand &  $\mathbf{13.2 \pm 8.3}$ & $24.7\pm11.8$ & $14.8\pm8.7$ & $33.6 \pm 9.6$ & $\mathbf{8.7 \pm 10.4}$  \\
Motor imagery: left vs. right hand  & $\mathbf{12.7\pm8.9}$ & $24.2\pm11.8$ & $15.3\pm10.5$ & $33.6 \pm 9.0$ & $\mathbf{7.4 \pm 9.8}$ \\
Motor execution: hands vs. feet   & $\mathbf{12.5\pm8.5}$ & $23.7\pm11.1$ & $15.2\pm9.6$ & $32.2 \pm 9.6$ & $\mathbf{10.0 \pm 10.7}$ \\
Motor imagery: hands vs. feet  & $\mathbf{12.9\pm8.9}$ & $25.5\pm10.2$ & $15.2\pm9.8$  & $33.4 \pm 9.2$ & $\mathbf{8.1 \pm 9.9}$ \\
\end{tabular} 
\end{table*}

We use a k-means clustering method, where the centroids are computed using either the Bures-Wasserstein mean, or the arithmetic or harmonic mean of the graph Laplacians, and the distance between graphs is measured according to the respective distance (Bures-Wasserstein distance, Frobenius norm of Laplacians, or Frobenius norm of pseudo-inverse Laplacians).
Moreover, we evaluate also the topological clustering methods in \citet{songdechakraiwut2021fast}, which compare persistence barcodes of the graphs using the Wasserstein distance. Note that the pure topological mean does not preserve the node alignment, and to address this, the authors propose a hybrid method of topological and arithmetic mean.

The classification performance of the different methods and for different values of $N_C$ (with fixed $p=0.2$) and respectively different values of $p$ (with fixed $N_C=5$) are summarized in Figure~\ref{fig:kmeans_class}.
\begin{figure}
\centering
\includegraphics[trim={0 5 0 10}, clip, width=\columnwidth]{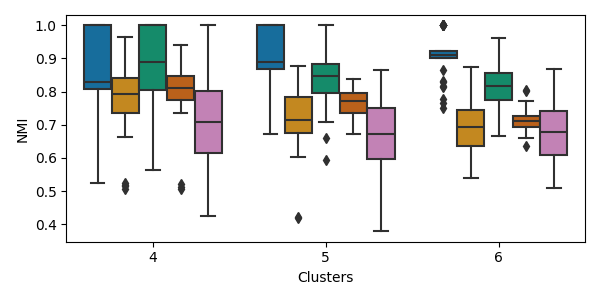} \hfill
\includegraphics[trim={0 10 0 10}, clip, width=\columnwidth]{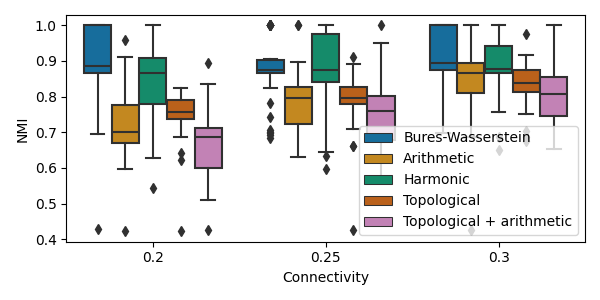}
\caption{Normalized mutual information of graph classification with k-means clustering for different values of number of classes (clusters), and inter-community edge probability (connectivity).}\label{fig:kmeans_class}
\end{figure}
One can see that for a small number of clusters, or respectively high connectivity within the communities, many of the tested methods perform similarly well.
However, the graph classification becomes difficult as the number of clusters increases, and as the connectivity within the communities decreases.
In these cases the Bures-Wasserstein method significantly outperforms the methods based on arithmetic and harmonic means.
Therefore, we can conclude that the distance in Definition~\ref{def:dist} captures similarity between the graphs well, and the mean graph \eqref{eq:graph_avg_BW} provides a useful representation of the cluster centroids. \looseness -1

\subsection{Classification of brain states} \label{subsec:brain}
We now show that the Bures-Wasserstein mean provides a useful centroid for graphs describing real world data. In particular, we utilize our methods to classify brain activity measurements from \citet{schalk2004bci2000, goldberger2000physiobank}\footnote{The dataset can be found at \url{https://mne.tools/0.18/generated/mne.datasets.eegbci.load_data.html}.}.
A subject is performing, or imagining, different types of motor tasks. Following a cue, the subject is contracting, or imagining to contract, the left hand vs. right hand, or the hands vs. feet.
Brain signals are measured from $64$ sensors placed around the head. 
For the brain reaction after each cue we construct a graph by assigning a node to each sensor.
We then use a 4 second interval after the cue to compute the envelope correlations between each set of sensors and assign it as an edge weight between the corresponding nodes.
We remove all edges that have an edge weight smaller than a threshold parameter, which is picked as large as possible while still resulting in a connected graph.
Each experiment trial consists of $15$ cues, which results in a set of $15$ graphs that are split into the two classes (either left vs. right hand, or hands vs. feet).
These sets of $15$ graphs are constructed for $109$ subjects, and $3$ experiment repetitions for each subject, resulting in a total of $327$ trials.
We compute the mean of the graphs for each of the two classes in the experiment, using the Bures-Wasserstein mean, the arithmetic and harmonic graph Laplacian mean, and the topological means in \citet{songdechakraiwut2021fast}.
We then compute the misclassification rate as the rate of graphs that are closest to the centroid of the wrong class with respect to the respective graph distance.
The resulting misclassification rate for the different sets of experiments, and using different types of graph means, are reported in Table~\ref{tab:brain_classification}.
In all experiments the Bures-Wasserstein mean provides centroids that describe the graphs in the two classes better than the centroids resulting from the arithmetic, harmonic, and topological means.
The only mean that outperforms the Bures-Wasserstein mean is the hybrid mean combining topological and arithmetic information, which was developed specifically for functional brain networks.
The fact that this hybrid mean performs significantly better than the pure arithmetic or topological mean indicates that a hybrid mean that combines the Bures-Wasserstein mean with topological information might drastically improve the clasification performance.
Developing a method for computing this mean is a promising direction of future research.
Note that the smooth graph signal representation for graphs is motivated by the standard assumption that real data samples are often well described by a normal distribution.
It is therefore reasonable to expect that the Bures-Wasserstein mean, which utilizes the smooth graph signal distributions, represents real world data well in many scenarios.

\subsection{Learning on multi-layer graphs} \label{subsec:node_class}
Finally, we show that the Bures-Wasserstein mean can serve as a regularizer for semi-supervised learning in a multi-layer graph.
A multi-layer graph is a set of graphs with the same set of nodes and different edges, which typically describe different types of interactions in a system.
We suggest the following classification method to learn unobserved labels from the structural information of the Bures-Wasserstein mean and some observed node labels, based on the state-of-the-art method introduced in \citet{mercado2019generalized}.
We consider that there are $k$ classes of nodes, and let $Y^{(r)} \in \mathbb{R}^N$ be defined by $Y^{(r)}_i = 1$ if the $i$-th node belongs to class $r$, and $Y^{(r)}_i = 0$ otherwise.
We then solve \looseness=-1
\begin{equation} \label{eq:node_class_opt}
    f^{(r)} = \argmin[f \in \mathbb{R}^N]  \| f - Y^{(r)} \|^2 + \rho f^T \bar{\mathcal{L}} f, 
\end{equation}  
where $\bar{\mathcal{L}}$ is some mean of a set of normalized graph Laplacians $\mathcal{L}_1,\dots,\mathcal{L}_m$ describing the multiple layers, and $\rho>0$ is a parameter that determines how much to rely on the given observations or the structure of the graphs.
Then the $i$-th node is assigned to the class $\text{arg} \max_r \{ f_i^{(r)} ; r=1,\dots,k \}.$
Here, we propose to define $\bar{\mathcal{L}}$ in the optimization problem \eqref{eq:node_class_opt} as the Bures-Wasserstein mean of the given graph Laplacians.
The set of 
normalized graph Laplacians are defined as in \citet{mercado2019generalized}, i.e.,
$\mathcal{L} = D^{-1/2} L D^{-1/2} $, where $L$ is the combinatorial graph Laplacian, and $D$ the degree matrix, as defined in Section~\ref{sec:main}.
We use the normalized Laplacians $\mathcal{L}_j$ for each layer $j$
here\footnote{In contrast to the previous examples, using normalized graph Laplacians is possible in this application, since we do not need to explicitly recover a graph representation $\bar G$ corresponding to the mean graph Laplacian $\bar{\mathcal{L}}$.}
to define the smooth graph signals $\mu_{G_j} \sim \mathcal{N}(0,\mathcal{L}_j^{\dag})$. 

We compare the performance of the node classification task utilizing the Bures-Wasserstein mean
with the methods in \citet{mercado2019generalized}, which utilize different types of power means to define an average graph Laplacian $\bar{\mathcal{L}}$ in \eqref{eq:node_class_opt}.
More precisely, the family of power means of the Laplacians $\mathcal{L}_j$, for $j=1,\dots,m$, is defined as $\bar{\mathcal{L}^p} = ( \sum_{j=1}^m \mathcal{L}_j^{p} )^{1/p} $. In particular, $p=1$ and $p=-1$ correspond to the arithmetic and harmonic mean, respectively. In addition to these two means, we also compare our method to the $p=-10$ power mean.
The regularization parameter is picked as the element from the set $\{0.1,1,10\}$ that gives the best results on the tested data. That is, $\rho=10$ for the arithmetic mean, $\rho=0.1$ for the harmonic and the power mean with $p=-10$ (as in \citet{mercado2019generalized}), and $\rho=1$ for the Bures-Wasserstein mean.

We consider multi-layer graphs constructed as in \citet{mercado2019generalized}, where a $10$-nearest neighbour graph is constructed for each layer of the real-world datasets \emph{3sources}, $\emph{BBC}$, \emph{BBCS} and \emph{Wikipedia}.
Several metrics for these datasets are presented in the supplementary material.
The average classification errors from $50$ trials for each dataset and with varying percentage of observed nodes are presented in Table~\ref{tab:node_classification}.
\begin{table}[t]
\setlength{\tabcolsep}{5pt}
\centering
\caption{Test error ($\%$) of node classification for different percentage of observed node labels.} \label{tab:node_classification} 
\begin{tabular}{c | c c c c c c c }
$p$ & $5\%$ & $10\%$ & $15\%$ & $20\%$ & $25\%$ & $30\%$ & $40\%$ \\
\midrule
\multicolumn{8}{l}{ \bf 3sources}\\
$1$  & $26.8$ & $23.4$ & $21.4$ & $16.7$ & $\mathbf{\underline{14.8}}$ & $\mathbf{\underline{13.3}}$ &  $\mathbf{\underline{11.2}}$  \\
$-1$ &  $\mathbf{\underline{22.5}}$ & $\mathbf{21.7}$ & $22.7$ & $19.3$ & $18.8$ & $17.9$  & $17.7$   \\
$-10$ & $31.2$ & $22.4$ & $\mathbf{20.6}$ & $\mathbf{16.2}$ & $\mathbf{14.9}$ &  $\mathbf{13.5}$ & $12.0$  \\ 
%\midrule
BW &  $\mathbf{24.9}$ & $\mathbf{\underline{21.3}}$ & $\mathbf{\underline{20.2}}$ & $\mathbf{\underline{15.9}}$ & $\mathbf{14.9}$ & $13.6$ &  $\mathbf{11.9}$ \\
\midrule 
 \multicolumn{8}{l}{\bf BBC} \\
$1$ 
& $22.6$& $17.1$ & $12.9$ & $10.6$ & $9.3$ & $\mathbf{8.2}$ &  $\mathbf{7.2}$ \\
$-1$ 
& $\mathbf{\underline{16.8}}$ & $\mathbf{\underline{11.7}}$ & $\mathbf{\underline{10.0}}$ & $\mathbf{9.5}$ & $\mathbf{9.0}$ & $8.5$ &  $8.2$  \\
$-10$
& $26.7$ & $16.4$ & $12.4$ & $10.5$ & $9.6$ & $8.6$ & $7.7$ \\ 
BW 
& $\mathbf{20.0}$ & $\mathbf{13.3}$ & $\mathbf{10.5}$ & $\mathbf{\underline{9.0}}$ & $\mathbf{\underline{8.4}}$ &  $\mathbf{\underline{7.4}}$ &  $\mathbf{6.9}$ \\
\midrule 
 \multicolumn{8}{l}{ \bf BBCS}  \\
$1$ & $16.1$& $13.3$ & $11.0$ & $8.8$ & $7.3$ & $6.3$ & $5.2$  \\
$-1$ & $\mathbf{\underline{12.3}}$ & $\mathbf{\underline{8.2}}$ & $\mathbf{\underline{6.5}}$ & $\mathbf{\underline{5.9}}$ & $\mathbf{\underline{5.4}}$ & $\mathbf{\underline{5.1}}$ &  $\mathbf{4.8}$ \\
$-10$ & $23.7$ & $13.8$ & $9.7$ & $7.7$ & $6.3$ & $5.6$ & $4.9$ \\ 
BW &  $\mathbf{15.8}$ & $\mathbf{10.4}$ & $\mathbf{8.0}$ & $\mathbf{6.5}$ & $\mathbf{5.6}$ & $\mathbf{5.2}$ &  $\mathbf{\underline{4.7}}$   \\
\midrule 
  \multicolumn{8}{l}{\bf Wikipedia} \\
$1$
& $59.6$& $53.6$ & $48.3$ & $44.5$ & $41.2$ & $38.4$ &  $35.3$ \\
$-1$ 
& $\mathbf{49.8}$ & $\mathbf{40.1}$ & $\mathbf{36.4}$ & $\mathbf{34.7}$ & $\mathbf{33.8}$ & $\mathbf{33.0}$ & $\mathbf{32.0}$  \\
$-10$ 
& $54.6$ & $43.3$ & $38.4$ & $35.9$ & $34.4$ & $33.2$ & $\mathbf{\underline{31.9}}$\\ 
BW 
 &  $\mathbf{\underline{49.0}}$ & $\mathbf{\underline{39.6}}$ & $\mathbf{\underline{36.0}}$ & $\mathbf{\underline{34.5}}$ & $\mathbf{\underline{33.5}}$ & $\mathbf{\underline{32.9}}$ &  $\mathbf{32.0}$ \\
\end{tabular}
\end{table}
One can see that the method based on the Bures-Wasserstein mean performs well in all tested scenarios, and is among the two best-performing methods in almost all settings. In contrast, the other types of averages are less reliable, and may perform well in some cases, but poorly in others.
For instance, note that \emph{3sources} is the smallest dataset, and thus the $10$-nearest neighbor graphs are the most strongly connected. This makes the structural information less informative, and none of the power means performs well for every tested percentage of observed node labels.
The \emph{Wikipedia} dataset contains the largest number of labels, which makes the node classification task on this dataset the most challenging, and a graph mean that accurately represents the structural information from the different layers becomes crucial. In this case especially our proposed Bures-Wasserstein mean outperforms the power means.
Thus, incorporating our proposed framework in this state-of-the-art method for semi-supervised learning, we are able to improve on its performance.

\section{CONCLUSION AND OUTLOOK}
We introduced a novel framework for defining the mean of a set of graphs, which utilizes the connection between graphs and their smooth graph signal distributions as well as the Wasserstein metric.
By combining these two concepts, the recovered mean graph can preserve structural information of the input graphs that is captured by the respective smooth graph signals.
By extending previous results for Wasserstein barycenters of non-degenerate Gaussians to a class of degenerate Gaussians, we proved the existence and uniqueness of the mean graph, and provide an effective iterative algorithm for finding it.
We evaluated the proposed approach on various machine learning tasks, including k-means clustering of structured graphs, classification of functional brain networks, and semi-supervised node classification in multi-layer graphs.
Our experimental results show that our approach achieves consistent performance that is competitive with existing baseline approaches and can improve state-of-the-art methods. 

Our results indicate that the proposed Bures-Wasserstein framework for graphs serves as a powerful tool for practical machine learning applications. Since defining means is a key ingredient of statistical analysis our approach may lead to various novel methods for graph machine learning.
For example, the Bures-Wasserstein mean together with the corresponding Fr\'echet variance can be used to build generative models for graphs.
More broadly, we see this work as a first step towards a general and flexible statistical framework for graph data. We anticipate that existing methods from optimal transport theory may be extended to admit statistical inference, regression, and high order interpolation methods for populations of graphs \citep{chen2018measure, chen2018optimal, karimi2020statistical, lambert2022variational}. %\looseness=-1

It should be noted that the methods presented in this work are currently limited to the setting of aligned graphs where a matching between the nodes in the given graphs is known. 
This is the case in various applications, e.g., multi-layer networks and brain networks estimated from EEG data as presented in this work.
However, in many other applications graphs are not aligned and may be of different sizes.
Another future direction of work is thus to develop computational methods that extend our approach to this different setting. \looseness=-1

\subsubsection*{Acknowledgements}
The authors thank the anonymous reviewers for useful comments. 
This work was supported by the Knut and Alice Wallenberg foundation under grant KAW 2021.0274.

\bibliography{references}

%%%%%%%%%%%%%%%%% Appendix %%%%%%%%%%%%%%%%%

\onecolumn

\linewidth\hsize \toptitlebar
\begin{center}
  \Large \bfseries Bures Wasserstein Means of Graphs:
\end{center} 
\begin{center}
\Large \bfseries Supplementary Materials
\end{center}
\bottomtitlebar

\vspace{20pt}

In the following we provide supplementary material for several aspects of the proposed Bures-Wasserstein graph mean framework.
In 
Section~\ref{sec:assp} we discuss Assumption~\ref{ass:psd},
in Section~\ref{sec:supp_filters} we explain how to extend the proposed frameworks to take into account graph filters, and in Section~\ref{sec:supp_exp} we provide details on the numerical experiments.
Finally, Section~\ref{sec:supp_proofs} contains the proofs of the theoretical results in the paper.

\section{DISCUSSION OF ASSUMPTION~\ref{ass:psd}} \label{sec:assp}

In this Section we discuss the class of graphs we consider in this work, namely the class of signed weighted graphs with Laplacian matrix that is positive semi-definite with only one zero eigenvalue, 
as stated in Assumption~\ref{ass:psd}.

Signed graphs are a powerful object to describe complex systems with positive as well as negative interactions. Negative edge weights could for instance model antagonistic relationships in social networks, or model inhibition and repression of expressions in gene regulatory networks.
In fact, Assumption 3.1 is ubiquitous for social consensus networks, as it is a sufficient condition for the system to converge to a consensus \citep{olfati2007consensus}.

For signed weighted graphs, the spectral properties are more difficult to analyze than for graphs with positive edge weights.
For instance, note that for graphs with only positive edge weights the multiplicity of the zero eigenvalue equals the number of disjoint components of the graph \citep{fiedler1973algebraic}.
Thus, if a graph with positive edge weights has a Laplacian with only one zero eigenvalue then it is connected.
However, this property does not extend to signed graphs: a connected signed graph may have a Laplacian with multiple zero eigenvalues.
On the other hand, it was shown that any signed weighted graph with a Laplacian that has a single zero eigenvalue is connected \citep[Lemma~1]{chen2016characterizing}.

Since we model graphs through their smooth graph signal distribution, a negative edge weight between two nodes is associated with smooth signals that have high dissimilarity between these two nodes. As a result, a negative edge could appear in the mean graph if the corresponding nodes are far away from each other in all input graphs, because in this case there is strong evidence for graph signals to be very dissimilar on the two nodes. As an example, consider the setting in the right panel of Figure~\ref{fig:example}. Here we find the mean of a set of graphs that are very densely connected: Most pairs of nodes are connected in at least one of the input graphs. However, the node pair $(0,4)$ is not connected in any of the input graphs, thus relative to all other node pairs these two nodes are very far from each other in all given data samples. From the given data it is thus expected that these two nodes are dissimilar, which is described by a negative edge weight in the mean graph.

Finally, we note that in contrast to our work, the previous works \citet{petric2019got, petric2020wasserstein} defined a similar distance to the Bures-Wasserstein distance, but only for unsigned graphs. 
However, as observed in the example of the right panel of Figure~\ref{fig:example}, the Bures-Wasserstein mean may exhibit negative edge weights even if the input graphs have only positive edge weights\footnote{It is worth noting that this is the case for many other graph means based on the Laplacian matrix, for instance power means with negative power (e.g., harmonic mean), and geometric means.}.
In fact, the graphs with positive edge weights do not define a geodesically complete space with respect to the Bures-Wasserstein distance.
From this we conclude that Assumption~\ref{ass:psd} defines a more sensible class of graphs for the Bures-Wasserstein distance.

\section{DETAILS ON BURES-WASSERSTEIN FILTER GRAPH MEANS} \label{sec:supp_filters}

In this Section we review some background on the filtered graph optimal transport distance introduced in \citet{petric2021fgot}, and discuss its application to the graph mean problem \eqref{eq:graph_avg}.

\subsection{Graph filter distance}

Recall that the Bures-Wasserstein distance for graphs is inspired from graph signal processing, a field that aims at generalizing classical signal processing concepts to graphs.
Using these tools, we can incorporate signal processing techniques, like filters, into the Bures-Wasserstein distance.
A graph signal $x\in \mathbb{R}^N$ filtered by the filter $g:\mathbb{R}^{N\times N} \to \mathbb{R}^{N\times N}$, which is acting on the graph Laplacian $L$, is of the form $g(L)x$.
Analogously to classical filters, low-pass graph filters emphasise the low frequencies of the graph signals, which correspond to slow variations of the signal between strongly connected nodes \citep{ortega2018graph}.
On the other hand, high-pass graph filters emphasise the large frequencies, i.e., higher variations between connected nodes.

Given a white Gaussian noise signal on the graphs nodes, denoted $w\in \mathbb{R}^N$, the filtered signal  $g(L) w$ follows the normal distribution $ \mathcal{N}(0, g(L)^2)$, see \citet{petric2021fgot}.
A class of filtered graph distances can be defined by utilizing the filtered graph signals in Definition~\ref{def:dist} as follows.
\begin{definition} \label{def:filterdist}
    Let $G_0$ and $G_1$ be two graphs with signed weighted Laplacian matrices $L_0$ and $L_1$. Given a graph filter $g:\mathbb{R}^{N\times N} \to \mathbb{R}^{N\times N}$, the corresponding graph filter distance between the two graphs is defined as 
$d_{BW}^g(G_0,G_1)  := \mathcal{W}_2(\mu_{G_{0}^g}, \mu_{G_{1}^g})$,
where $\mu_{G_{j}^g} \sim \mathcal{N}(0, g(L_j)^2)$ for $j=0,1$.
\end{definition}

Hence, essentially different graph filter distances correspond to different embeddings of the graphs into the space of signal distributions.
Utilizing a low-pass filter in Definition~\ref{def:filterdist} results in a distance that captures mainly differences in the global graph behavior, and a high-pass filter measures predominantly local differences in the graphs, as observed in \citet{petric2021fgot}.
Finally, note that the filter $g(L)= L^{\dag/2}$ gives rise to the smooth graph signal distribution $ \mathcal{N}(0, L^\dag)$.
Thus, utilizing this filter in the Bures-Wasserstein filter distance in Definition~\ref{def:filterdist} results in the standard Bures-Wasserstein distance for graphs presented in Definition~\ref{def:dist}.

\subsection{Bures-Wasserstein filter graph mean}

We now consider using these filter graph distances in the Bures-Wasserstein mean problem \eqref{eq:graph_avg_BW}. By emphasising different types of spectral information of the input graphs, we expect to retrieve a mean graph that preserves the corresponding type of spectral properties of the input graphs.
We define the Bures-Wasserstein filter graph mean as the solution of
\begin{equation} \label{eq:graph_avg_BW_filtered}
\bar{G} = \argmin[G \in \mathcal{G}] \sum_{j=1}^m \lambda_j d_{BW}^g(G,G_j)^2  = \argmin[G \in \mathcal{G}] \sum_{j=1}^m \lambda_j \mathcal{W}_2(\mu_{G}^g, \mu_{G_j^g})^2.
\end{equation}
In order to recover a mean graph $G$ from the Wasserstein barycenter $\mu_G^g$, we require that there is a one-to-one mapping from the set of graphs in $\mathcal{G}$ to the set of signal distributions. In other words, the graph filter is required to be bijective on the set of Laplacian matrices that satisfy Assumption~\ref{ass:psd}.
With this assumption we can generalize the existence and uniqueness results from Section~\ref{sec:main}.
\begin{corollary}
Let $G_j$, $j=1,\dots,m$ be
graphs with signed weighted Laplacians $L_j$ that satisfy Assumption~\ref{ass:psd}, and let $g:\mathbb{R}^{N\times N} \to \mathbb{R}^{N\times N}$ be a bijective mapping on the set of matrices that satisfy Assumption~\ref{ass:psd}.
Then the Bures-Wasserstein filter mean graph \eqref{eq:graph_avg_BW_filtered} of these graphs has the signed weighted Laplacian matrix $L=g^{-1}(S^{1/2})$, where $S$ is the unique positive semi-definite symmetric solution to 
\begin{equation*} \label{eq:bary_eq_degen_filter}
S = \sum_{j=1}^m \lambda_j \left( S^{1/2} g(L_j)^2 S^{1/2} \right)^{1/2}
\end{equation*}
that satisfies $\text{range}(S)=\mathcal{R} = (\text{span}\{\mathbf{1}_N\})^\perp$. 
\end{corollary}
\begin{proof}
    The result follows directly from Theorem~\ref{thm:bary_unique}, since the filter $g$ is bijective.
\end{proof}

\begin{corollary}
Let $G_0$ and $G_1$ be two graphs with signed weighted Laplacians $L_0$ and $L_1$ that satisfy Assumption~\ref{ass:psd}, and let $g:\mathbb{R}^{N\times N} \to \mathbb{R}^{N\times N}$ be a bijective mapping on the set of matrices that satisfy Assumption~\ref{ass:psd}.
Then the Bures-Wasserstein filter mean of these two graphs with weights $1-t$ and $t$ for $t\in(0,1)$ has the signed weighted graph Laplacian $L_t = g^{-1}(S_t^{1/2})$, where 
\begin{equation*} \label{eq:interpolation}
S_t = g(L_0)^\dag \left( (1-t) g(L_0)^2 +t \left( g(L_0) g(L_1)^2 g(L_0) \right)^{1/2} \right)^2 g(L_0)^\dag.
\end{equation*}    
\end{corollary}
\begin{proof}
   The result follows directly from Theorem~\ref{thm:geodesic}, since the filter $g$ is bijective. 
\end{proof}

Finally, we can generalize Algorithm~\ref{alg:bary} to taking into account graph filters. The resulting method is presented in Algorithm~\ref{alg:bary_filter}. 
\begin{algorithm}[tb]
\begin{algorithmic}
\State {\bf Given:} Matrices $L_1,\dots,L_m \in \mathbb{R}^{N\times N}$ satisfying Assumption~\ref{ass:psd}, and initial SPD matrix $S\in \mathbb{R}^{N\times N}$
\State $\displaystyle{\Sigma_j \gets g \left(L_j  + \frac{1}{N} \mathbf{1}_{N\times N} \right)^2}$ for $j=1,\dots,m$
\While{Not converged}
\State $\displaystyle{ S \gets S^{-1/2} \left( \sum_{j=1}^m \lambda_j ( S^{1/2} \Sigma_j S^{1/2} )^{1/2} \right)^2 S^{-1/2},}$ 
\EndWhile
\State {\bf Return:} $\displaystyle{L \gets g^{-1}(S^{1/2}) - \frac{1}{N} \mathbf{1}_{N\times N} }$
\end{algorithmic}
\caption{Bures-Wasserstein filter mean of graphs.} \label{alg:bary_filter}
\end{algorithm}

\subsection{Experimental results} \label{subsec:filt_exp}

We now illustrate the behavior of the Bures-Wasserstein filter mean graphs with respect to several filters.
In particular, we study the graph filters summarized in the following table.
\begin{center}
    \begin{tabular}{l | l}
%   $g(L)$  & \\ \midrule
   $g(L)= L^\dag$ & most low-pass \\
   $g(L)=L^{\dag/2}$ & standard Bures-Wasserstein \\
   $g(L)=L^{1/2}$ &  \\
   $g(L)=L$ & most high-pass\\
\end{tabular}
\end{center}

We present the results of some of the experiments considered in Section~\ref{sec:exp} utilizing these filters in the Bures-Wasserstein filter graph mean problem \eqref{eq:bary_eq_degen_filter}.
The respective optimization problems are solved using Algorithm~\ref{alg:bary_filter}.

\subsubsection{Graph fusion}
First, we consider the graph fusion experiment that is set up as described in Section~\ref{subsec:fusion}.
Table~\ref{tab:graph_fusion_filters} shows how the graph metrics are preserved for different types of Bures-Wasserstein means.
\begin{table}
\caption{Difference of the perturbed graphs' mean to the original graph with respect to several metrics. Compare Table~\ref{tab:graph_fusion}.} \label{tab:graph_fusion_filters}
\begin{tabular}{ l | c c c c c}
 & Laplacian & Covariance & degree centrality & modularity & participation coeff. \\
\midrule
Arithmetic &  $1.85 \pm 0.10$ & $0.26 \pm 0.18$ & $\mathbf{0.026 \pm 0.003}$ & $0.015 \pm 0.003$ & $0.37 \pm 0.13$ \\
Harmonic &  $2.87 \pm 0.19$ & $\mathbf{0.13 \pm 0.10}$ & $0.047 \pm 0.004$ & $0.012 \pm 0.002$ & $0.26 \pm 0.09$ \\
Geometric &  $7.25 \pm 0.36$ & $0.38 \pm 0.17$ & $0.103\pm 0.010$ & $0.015 \pm 0.009$ & $0.40 \pm 0.12$ \\ \midrule
$g(L) = L^\dag$ & $2.39 \pm 0.13$ & $ \mathbf{0.13 \pm 0.08 }$ & $ 0.039 \pm 0.003 $ & $ \mathbf{0.009 \pm 0.002}$ & $\mathbf{0.22 \pm 0.08}$ \\
$g(L) = L^{\dag/2}$ &  $2.26 \pm 0.11$ & $0.15 \pm 0.12$ & $0.036 \pm 0.002$ & $\mathbf{0.011 \pm 0.003}$ & $\mathbf{0.25 \pm 0.09}$ \\
$g(L)=L^{1/2}$ & $\mathbf{1.81 \pm 0.10}$ & $ 0.21 \pm 0.13$ & $\mathbf{0.026 \pm 0.003}$ & $0.012 \pm 0.002$ & $0.30 \pm 0.11$ \\
$g(L)= L$ & $\mathbf{1.68 \pm 0.10}$ & $ 0.23 \pm 0.14$ & $\mathbf{0.025 \pm 0.003}$ & $0.012 \pm 0.002$ & $0.30 \pm 0.12$ \\
\end{tabular}
\end{table}
As expected, the low-pass filters preserve mostly global graph metrics, such as modularity, and the covariance matrix, whereas the high-pass filters preserve local behavior, such as degree centrality and the Laplacian matrix.
These findings highlight the flexibility and versatility of our proposed graph filter means, providing practitioners with an easy-to-use framework that can be readily adapted to different applications, depending on the relative importance of local versus global graph behavior.

\subsubsection{Semi-supervised learning in multi-layer graphs} \label{subsec:nodeclass_filters}

Next, we repeat the node-classification experiments in Section~\ref{subsec:node_class}, and include also the low-pass filter $g(L)=L^\dag$ and the high-pass filter $g(L)=L^{1/2}$.
The results are presented in Figure~\ref{fig:nodeclass}.
 \begin{figure}[tb]
 \centering
  \includegraphics[trim={0 0 0 0}, clip, width=0.49\textwidth]{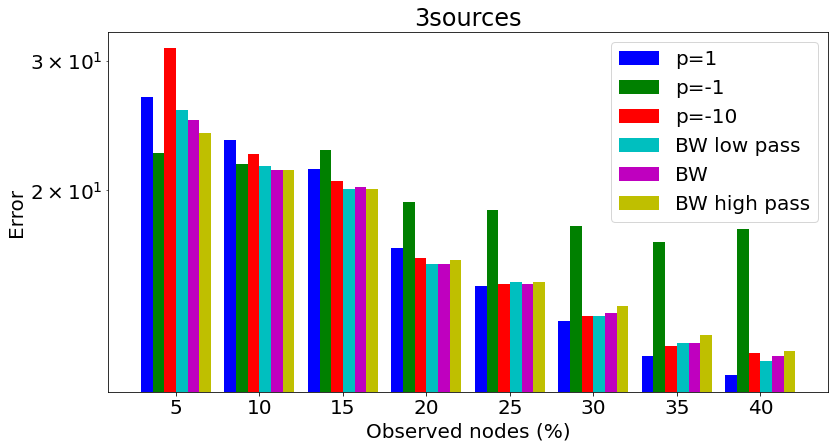} \hfill
\includegraphics[trim={0 0 0 0}, clip, width=0.49\textwidth]{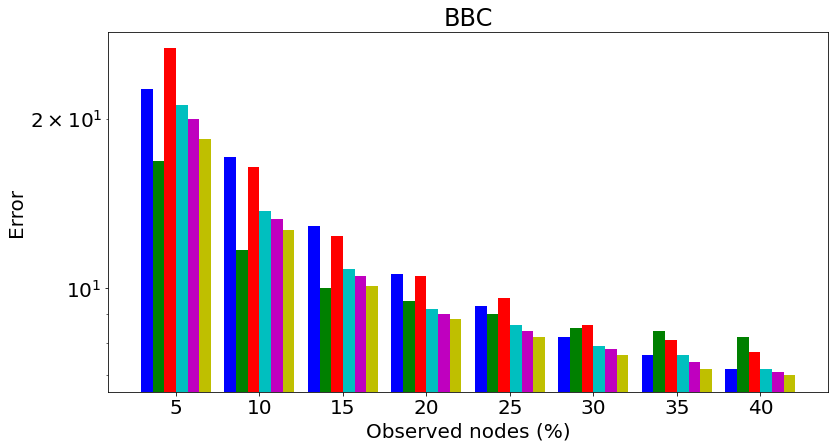}\\
\includegraphics[trim={0 0 0 0}, clip, width=0.49\textwidth]{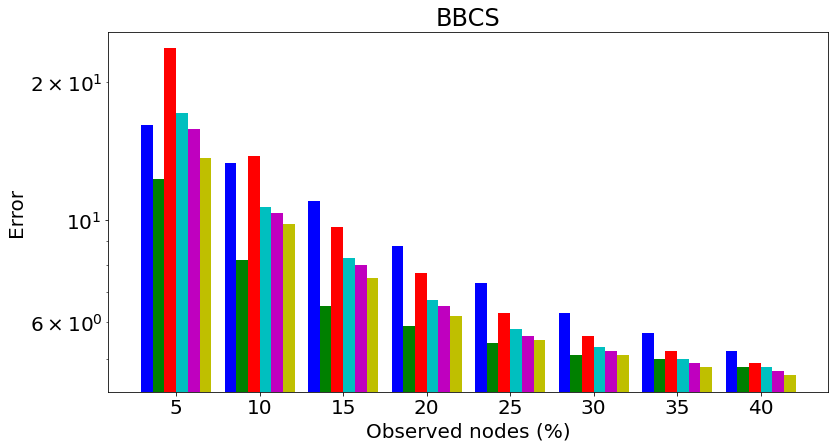} \hfill
\includegraphics[trim={0 0 0 0}, clip, width=0.49\textwidth]{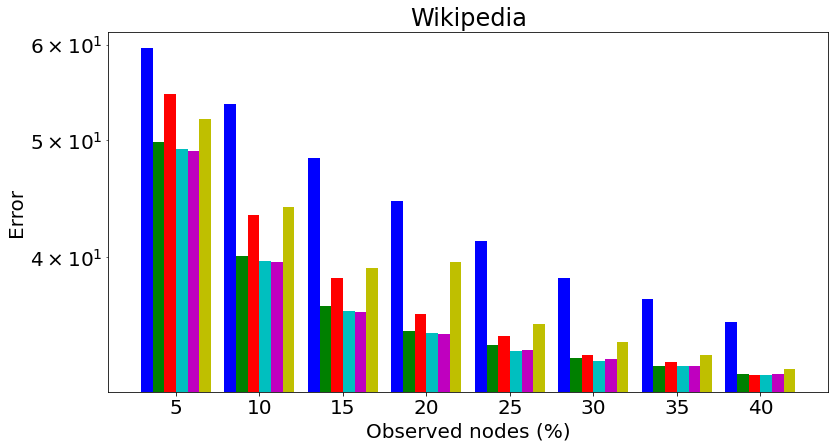}
\caption{Test error of node classification for different percentage of observed node labels using several types of graph means and distances. Here, \emph{BW low pass} and \emph{BW high pass} correspond to the filter distance with $g(L)= L^\dag$ and $g(L)= L^{1/2}$, respectively. }\label{fig:nodeclass}
\end{figure}
Overall, we observe that the tested graph filters exhibit strong performance in most experiments.
For a detailed comparison of the datasets, we refer to Section~\ref{subsec:app_nodeclass}.
As a general trend, the high-pass filter mean tends to perform very well, particularly on the more homogeneous datasets such as \emph{BBC} and \emph{BBCS}.
However, in the more challenging datasets \emph{3sources} and \emph{Wikipedia} the low-pass filter mean may outperform the high-pass filter mean.
Note that all tested graph means have the highest error rates for the \emph{Wikipedia} dataset.
Especially for this challenging dataset the more low-pass filter means demonstrate significant improvements over the high-pass filter mean.

As already observed in Section~\ref{subsec:node_class}, none of the power means performs well in all scenarios, which is in contrast to the Bures-Wasserstein means.
For instance, the harmonic mean (i.e., the power mean with $p=-1$) performs very poorly on the \emph{3sources} dataset when more than $25\%$ of the nodes are observed, and the arithmetic mean (i.e., the power mean with $p=1$) does not yield satisfactory results on the \emph{Wikipedia} dataset.
These observations lead us to conclude that the Bures-Wasserstein filter means offer a reliable family of graph means, capable of delivering consistent performance across various datasets.

\section{DETAILS ON EXPERIMENTS} \label{sec:supp_exp}

In this Section we present details on the numerical experiments in Section~\ref{sec:exp}.
Specifically, in the following subsections we describe the baseline methods, discuss computational aspects of our proposed Algorithm~\ref{alg:bary}, and we provide more information on the data for the k-means clustering example in Section~\ref{subsec:kmeans}, and the semi-supervised node-classification problem in \ref{subsec:node_class}.

\subsection{Baseline methods}

In this section we briefly describe the baseline methods considered in Figure~\ref{fig:example} and Section~\ref{sec:exp}.

Various means for scalars can be generalized to symmetric positive definite matrices \citep{lim2012matrix}, and have been utilized for graph Laplacian matrices \citep{mercado2019generalized, mercado2019spectral, el2020orthonet}.
These means parameterize the class of graphs $\mathcal{G}$ in the graph mean problem \eqref{eq:graph_avg} by the graph Laplacian matrices of size $\mathbb{R}^{N \times N}$. 
For instance, when using the Frobenius norm as the distance $d(\cdot,\cdot)$ in \eqref{eq:graph_avg}, the mean of a set of graph Laplacians $L_1,\dots,L_m$ is their arithmetic mean $\sum_{j=1}^m \lambda_j L_j$.
Similarly, the harmonic mean is defined as $ \left( \sum_{j=1}^m \lambda_j L_j^{\dag} \right)^{\dag}$. 
These two means are special cases of the so-called power means introduced in Section~\ref{subsec:node_class}.

The generalization of the geometric mean from scalars to symmetric positive definite matrices is less straight-forward, and cannot be expressed in closed form \citep{lim2012matrix}. 
In this work we compute the geometric mean iteratively as described in \citet[Section~IV.C]{el2020orthonet}.

In Sections~\ref{subsec:kmeans} and \ref{subsec:brain} we evaluate also a topological clustering method that compares persistence barcodes of the graphs using the Wasserstein distance \citet{songdechakraiwut2021fast}.
The retrieved mean thus compares topological features of the input graphs, but does not preserve the node alignment. In order to address this the authors propose a hybrid method of topological and arithmetic mean, which preserves node correspondences. In this work, we utilise a weighting factor of $\lambda = \frac{1}{2}$ for the topological and arithmetic component of this hybrid mean.

\subsection{Computational aspects}

In this section, we examine the behavior of our proposed computational method outlined in Algorithm~\ref{alg:bary}. As highlighted in Section~\ref{sec:main}, the fixed-point iteration employed in Algorithm~\ref{alg:bary} is based on a similar iteration used for the Wasserstein barycenter problem \eqref{eq:barycenter} with normal distributions as input \citep{alvarez2016fixed}. Consequently, this fixed-point iteration converges towards the unique solution of the matrix equation \eqref{eq:bary_eq_degen}, which characterizes the Bures-Wasserstein mean graph.
Let $S^{(n)} \in \mathbb{R}^{N\times N}$ denote the outcome of the fixed-point iteration after $n$ iterations. In our experiments, we adopt the stopping criterion $|S^{(n)} - S^{(n-1)}|<10^{-5}$, which is typically achieved within a few iterations ($n\le 5$). However, it is important to note that to our knowledge there is currently no theoretical analysis available regarding the convergence rate.

We observe that each iteration of the fixed-point method in Algorithm~\ref{alg:bary} involves computing one matrix inverse and performing $m+1$ square roots of symmetric positive definite matrices. As a result, the computational complexity of one fixed-point iteration in the algorithm is $\mathcal{O}(mN^3)$. It is worth mentioning that the factor $m$ can be effectively addressed by parallelizing the summation in Algorithm~\ref{alg:bary}.
The cubic dependence in $N$ may be addressed by using approximations of the matrix inverse, and matrix roots, or by exploiting graph sparsity.
Investigating ways to enhance the computational efficiency of Algorithm~\ref{alg:bary} is an important direction for future work.

\begin{wrapfigure}{r}{0.43\textwidth}
\vspace{-10pt}
\centering
\begin{minipage}{.4\textwidth}
\begin{framed}
    \centering
{\bf{Data samples}}

\vspace{5pt}

\begin{minipage}{.2\textwidth}

Class 1\\

\vspace{10pt}
Class 2\\

\vspace{10pt}
Class 3\\

\vspace{10pt}
Class 4\\

\vspace{10pt}
Class 5 

\end{minipage}
\begin{minipage}{.75\textwidth}

    \includegraphics[width=\textwidth]{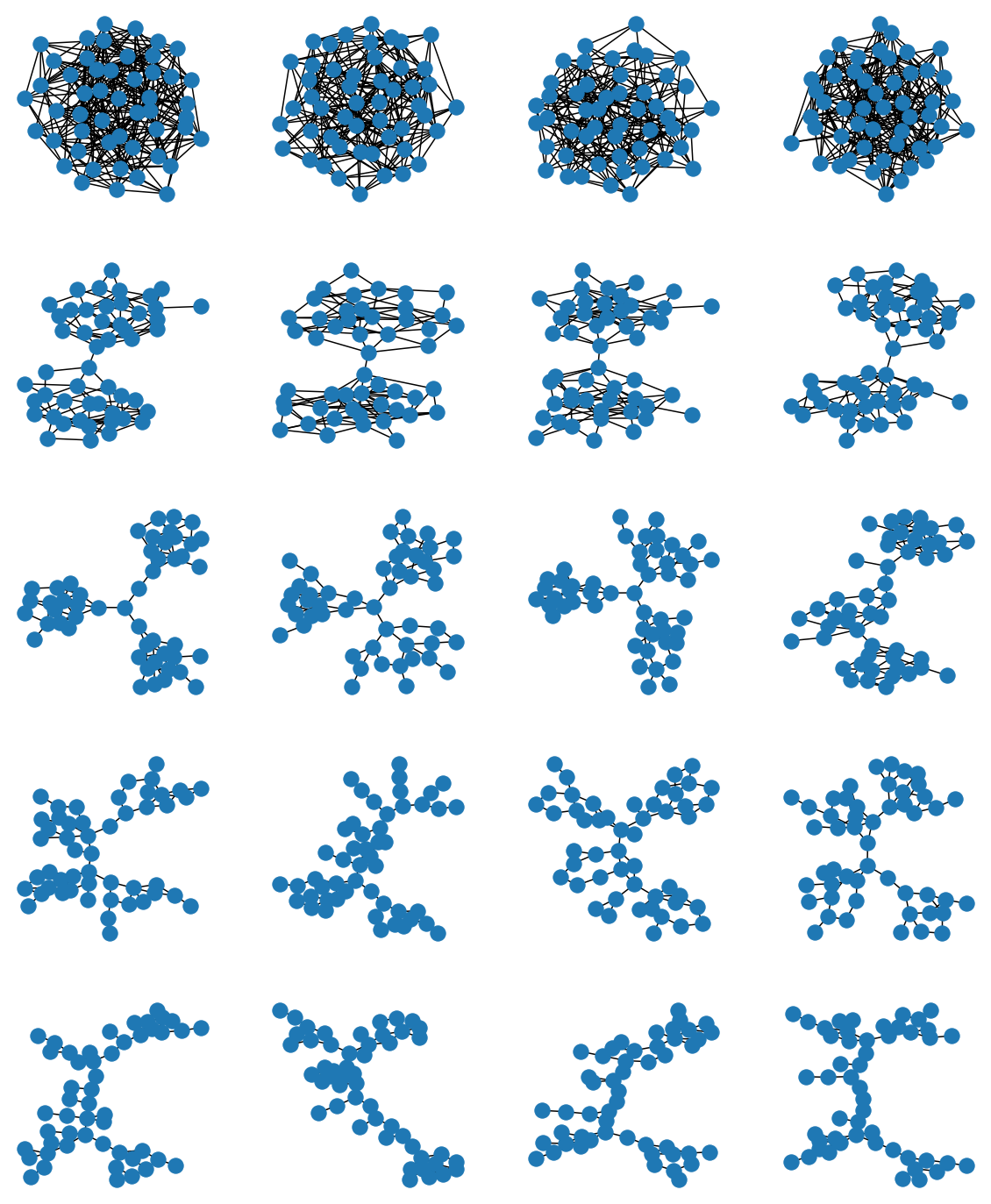}

  \end{minipage}
    \end{framed}     
\end{minipage}
\caption{Some data samples of the k-means clustering example in Section~\ref{subsec:kmeans} with $k=N_C=5$ clusters and inter-community edge probability $p=0.2$. } \label{fig:kmeans_samples}
\end{wrapfigure}

\subsection{Details on k-means clustering in Section~\ref{subsec:kmeans}}

The task in the k-means clustering example in Section~\ref{subsec:kmeans} is to classify structured graphs with different numbers of communities.
In Figure~\ref{fig:kmeans_samples} we show some graph samples from the dataset.
More precisely, we consider the setting with $N_C=5$ classes, and showcase four samples from each of these classes of graphs.
Here, different classes are defined by the number of communities, and communities are arranged on a line. There are two different connectivity patterns between neighbouring communities, which are each picked with probability $1/2$.
Each of these two patterns constructs one edge between the communities, between a different set of nodes.

\begin{figure}
\centering
\begin{minipage}{.7\textwidth}
\begin{framed}
    \centering
{\bf{Centroids}}
\vspace{5pt}

    \begin{minipage}{0.02\textwidth}
\rotatebox[origin=c]{90}{Iteration 2 \ $\longleftarrow$ \ Iteration 1 \ $\longleftarrow$ \ Iteration 0}
\end{minipage}
\hfill
\begin{minipage}{0.95\textwidth}

    \includegraphics[width=\textwidth]{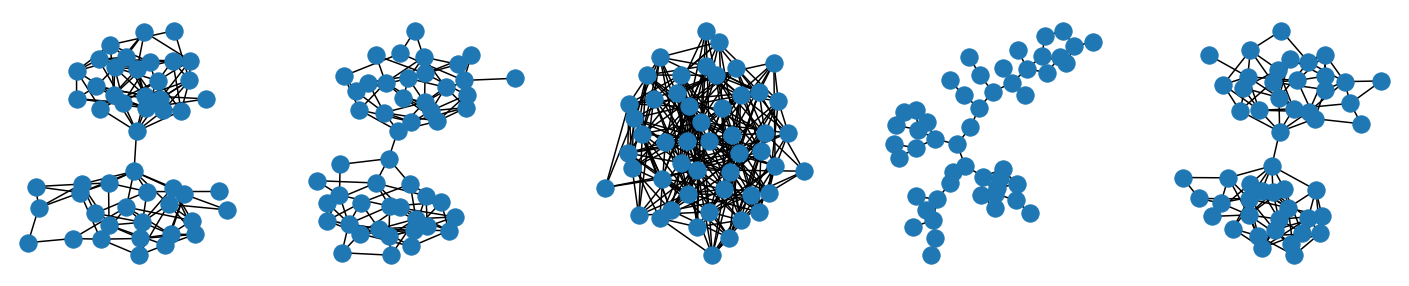}

    \vspace{6pt}

    \includegraphics[width=\textwidth]{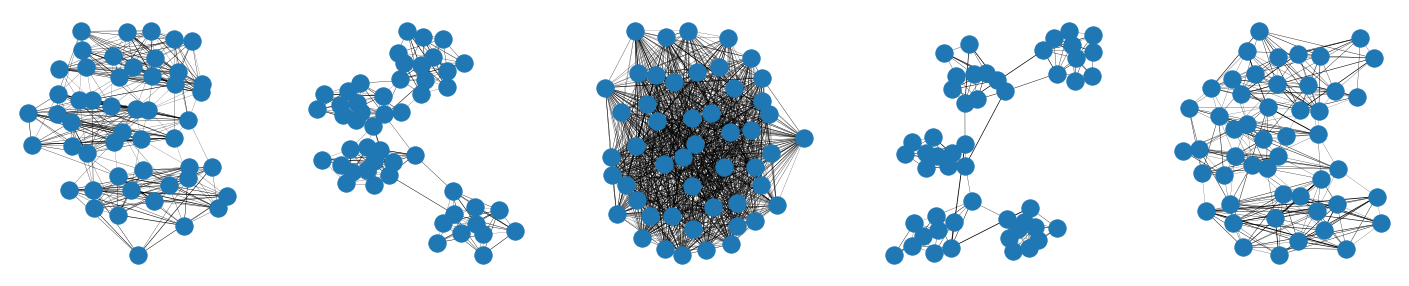}

    \vspace{6pt}

    \includegraphics[width=\textwidth]{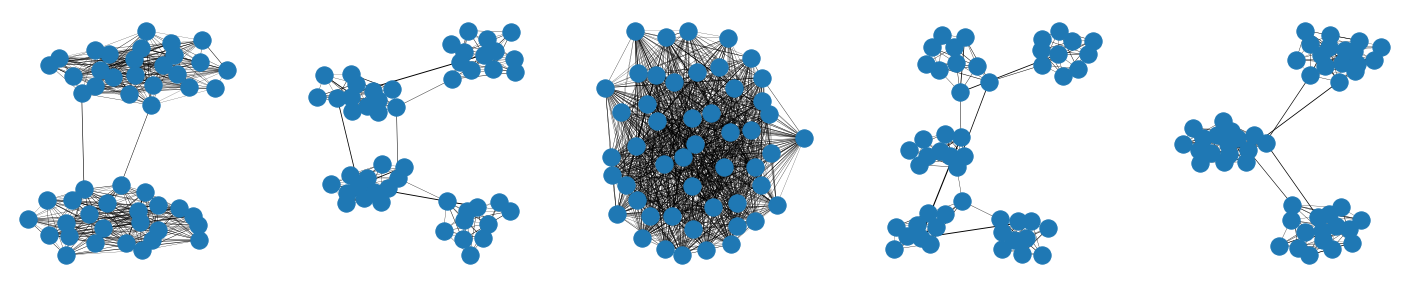}

\end{minipage}

\vspace{6pt}

 \hspace{8pt} Cluster 1 \hspace{17pt} Cluster 2 \hspace{17pt} Cluster 3  \hspace{17pt} Cluster 4 \hspace{17pt} Cluster 5
    \end{framed}
\end{minipage}
    \caption{Evolution of the centroids for one trial of the k-means clustering example with data as illustrated in Figure~\ref{fig:kmeans_samples}.}
    \label{fig:kmeans_centroids}
\end{figure}

Figure~\ref{fig:kmeans_centroids} illustrates the evolution of centroids computed by the k-means classifier visually.
The method is initialized with five graphs picked randomly from the dataset. In this example the k-means clustering method converges in two iterations.
We note here that the five final centroids are graphs, where the number of communities ranges from one to five. Moreover, we can see that the two distinct connectivity patterns between neighboring communities are both captured by the centroids. 
We thus conclude that the final centroids provide meaningful representations of the graphs in each of the five classes.

\subsection{Details on semi-supervised learning in multi-layer graphs in Section~\ref{subsec:node_class}} \label{subsec:app_nodeclass}

For the semi-supervised node-classification problem in Section~\ref{subsec:node_class}, we use four commonly used datasets: \emph{3sources}, \emph{BBC}, \emph{BBCS}, and \emph{Wikipedia}.
Here, layers correspond to different features of the data, and for each layer a k-nearest neighbour graph is constructed based on the Pearson linear correlation between nodes, as described in \citet{mercado2019generalized}.
Thus, nodes with high correlation are close to each other in the resulting graphs.
Table~\ref{tab:summary_multilayer} summarizes some characteristics of the resulting multi-layer graphs.
\begin{table} \caption{Characteristics of the multi-layer graphs used in the semi-supervised node classification experiments in Section~\ref{subsec:node_class}.} \label{tab:summary_multilayer}
\centering
%\small
\begin{tabular}{l | c | c | c | c | c | c  }
 & Layers & Nodes   & Labels & Edges & \multicolumn{2}{c}{Mean degree centrality } \\
  & & &  & layers & average & layers \\
 \midrule
3sources &  $3$ & $169$ & $6$ & $1084$, $1139$, $1188$ & $13.4$ & $12.8$, $13.5$, $14.1$ \\
BBC  & $4$ & $685$ & $5$ & $4806$, $4775$, $4816$, $4874$ & $14.1$ & $14.0$, $14.0$, $14.1$, $14.2$ \\
BBCS   & $2$ & $544$  & $5$ & $3803$, $3855$ & $14.1$ & $14.0$, $14.17$\\
Wikipedia  & $2$ & $693$  & $10$ & $5260$, $5038$ & $14.9$ & $15.2$, $14.5$ \\
\end{tabular} \vspace{5pt}
\end{table}
We observe that the datasets \emph{BBC} and \emph{BBCS} exhibit a higher level of homogeneity between the layers in terms of the number of edges and average degree centrality.
This finding explains the consistent and relatively similar performance of all tested methods on these datasets, as discussed in Section~\ref{subsec:node_class}.
However, the more complex datasets, such as \emph{3sources} and \emph{Wikipedia}, exhibit greater heterogeneity between the layers, leading to more pronounced differences in the performance of the tested methods.
Moreover, these datasets present additional challenges.
For instance, in the \emph{3sources} dataset, the graphs have the smallest number of nodes, but have similar degree centrality, and thus their nodes are most strongly connected,
resulting in less informative structural information.
On the other hand, the \emph{Wikipedia} dataset involves 10 different types of node labels, making the classification task particularly challenging.
Remarkably, as demonstrated in Sections~\ref{subsec:node_class} and \ref{subsec:nodeclass_filters}, our proposed Bures-Wasserstein mean consistently outperforms state-of-the-art methods, particularly on these challenging datasets.

\section{PROOFS} \label{sec:supp_proofs}

In this Section we present the proofs of the main results in Section~\ref{sec:main}.

\subsection{Proof of Theorem~\ref{thm:bary_unique}} 

We now present a proof of the the main result of the paper, which asserts the existence and uniqueness of a solution to the Bures-Wasserstein mean problem for graphs.
To prove the result, we first show that the involved optimal transport problems on $\mathbb{R}^N$ can be projected into the space $\mathcal{R} = (\text{span}\{\mathbf{1}_n \})^\perp$ without changing the optimal solution.

\begin{lemma} \label{lem:ot_proj}
 Let $\mu_0,\mu_1$ be two probability measures with support on a $d$-dimensional subspace $ \mathcal{R} \subset \mathbb{R}^N$, and let $P:\mathbb{R}^N \to \mathbb{R}^d $ be a mapping that is distance preserving between $ \mathcal{R}$ and $\mathbb{R}^d$. 
Denote $\mathcal{W}_2^{\mathcal{X}}(\cdot,\cdot)$ the Wasserstein distance in space $\mathcal{X}$.
Then it holds that
\begin{equation} \label{eq:prop_subspace}
 \mathcal{W}_2^{\mathbb{R}^N}(\mu_0,\mu_1) = \mathcal{W}_2^{\mathbb{R}^d}( P_\# \mu_0, P_{\#} \mu_1).
\end{equation}
\end{lemma}
\begin{proof}
Since the measures $\mu_0$ and $\mu_1$ have support in $\mathcal{R}$ it follows that any feasible transport plan in $\Pi(\mu_0, \mu_1)$ has support in $\mathcal{R} \times \mathcal{R}$.
Thus, the isometry between $\mathcal{R}$ and $\mathbb{R}^d$ given by the mapping $P$ infers a bijection $P \times P$ between the feasibility sets $\Pi(\mu_0, \mu_1)$ and $\Pi(P_\# \mu_0, P_\# \mu_1)$ of the optimization problems on the left hand side and right hand side of equation \eqref{eq:prop_subspace}, respectively.
More precisely, 
let $P^{-1}:\mathbb{R}^d \to \mathcal{R}$ be the inverse of $P$ restricted to $\mathcal{R}$.
Then, for any measurable set $A\subset \mathcal {R}$ it holds that
$$ \pi(A, \mathbb{R}^N) = \mu_0(A)  \implies  \left( \left( P\times P \right)_\# \pi \right) \left( P(A) \times \mathbb{R}^d \right) = \left( P_\# \mu_0 \right)  ( P(A) ), $$
and for any measurable set $\hat A \subset \mathbb{R}^d$ and measure $\hat \pi \in \mathbb{R}^d \times \mathbb{R}^d$ it holds that
$$ \hat \pi ( \hat A, \mathbb{R}^d ) = (P_\# \mu_0) ( \hat A) \implies \left( \left( P^{-1}\times P^{-1} \right)_\# \hat\pi \right) \left( P^{-1}( \hat A) \times \mathbb{R}^N \right) = \mu_0   ( P^{-1}(A) ) . $$
Similar relationships hold for the second marginal of the product measures and $\mu_1$.
Hence it follows that
$$ \pi \in  \Pi(\mu_0, \mu_1) \iff \left( P\times P \right)_\# \pi \in \Pi(P_\# \mu_0, P_\# \mu_1).   $$

    Finally, note that for any $\pi \in \Pi(\mu_0, \mu_1)$, where $\mu_0$ and $\mu_1$ have support in $\mathcal{R}$, since $\pi$ has support in $\mathcal{R} \times \mathcal{R}$, it holds that
    \begin{equation*}
        \begin{aligned}
            \int_{\mathbb{R}^N \times \mathbb{R}^N } \| x-y \|_2^2 \ d \pi(x, y) & = \int_{\mathcal{R} \times \mathcal{R}} \|x-y\|_2^2 \ d \pi(x, y) \\
            &= \int_{\mathbb{R}^d \times \mathbb{R}^d } \|x-y\|_2^2 \ d \left(( P \times P )_\# \pi \right) (x,y). 
        \end{aligned}
    \end{equation*}
    Thus, the objective values of the optimization problems in equation \eqref{eq:prop_subspace} coincide for $\pi$ on the left hand side of \eqref{eq:prop_subspace} and $( P \times P )_\# \pi$ on the right hand side of \eqref{eq:prop_subspace}.
    This concludes the equality of the Wasserstein distances in \eqref{eq:prop_subspace}.
\end{proof}

We can now prove Theorem~\ref{thm:bary_unique}.

\begin{proof}[Proof of Theorem~\ref{thm:bary_unique}]

We prove the result by considering the Wasserstein barycenter problem \eqref{eq:barycenter} of the measures $\mu_{G_j} \sim \mathcal{N}(0, L_j^\dag)$ for $j=1,\dots,m$.
It turns out that its solution in fact defines a unique solution to the Bures-Wasserstein mean graph problem \eqref{eq:graph_avg_BW}.

Since the Laplacians $L_j$ satisfy Assumption~\ref{ass:psd}, they all have the same nullspace $\text{span}\{\mathbf{1}_N\}$ and thus, the support of all the probability measures $\mu_{G_j} \sim \mathcal{N}(0, L_j^\dag)$ for $j=1,\dots,m$ is the subspace $\mathcal{R} = (\text{span}\{\mathbf{1}_N\})^\perp$ of $\mathbb{R}^N$.
Consider the orthogonal projection $P^\perp:\mathbb{R}^N \to \mathcal{R}$, and note that it holds $\mathcal{W}^{\mathbb{R}^N}_2(\mu,\mu_{G_j}) \geq \mathcal{W}^{\mathbb{R}^N}_2(P^\perp_{\#\mu},\mu_{G_j})$ for any probability measure $\mu$ on $\mathbb{R}^N$.
Thus, the Wasserstein barycenter $\mu$ of the measures $\mu_{G_j}$, for $j=1,\dots,m$, also has support $\mathcal{R}$.

Consider a matrix $U \in \mathbb{R}^{N \times N-1}$ whose columns form an orthonormal basis of the subspace $\mathcal{R}$ in $\mathbb{R}^N$.
Then, the mapping $P:\mathbb{R}^N \to \mathbb{R}^{N-1}$ defined as $v \mapsto U^T v$ is distance preserving between $\mathcal{R}$ and $\mathbb{R}^{N-1}$.
Hence, applying Lemma~\ref{lem:ot_proj} to every term in the sum of the Wasserstein barycenter problem \eqref{eq:barycenter} yields that
\begin{equation} \label{eq:bary_equi}
\minwrt[\mu ] \sum_{j=1}^m \lambda_j \mathcal{W}^{\mathbb{R}^N}_2(\mu,\mu_{G_j})^2  = \minwrt[ \bar\mu]  \sum_{j=1}^m \lambda_j \mathcal{W}^{\mathbb{R}^{N-1}}_2(\bar\mu,P_\#\mu_{G_j}) ^2,
\end{equation}
where $\mu$ is a measure on $\mathbb{R}^{N}$ and $\bar{\mu}$ is a measure on $ \mathbb{R}^{N-1}$.
Moreover, the minimizers $\mu$ and $\bar \mu$ of the left hand side and ride hand side of \eqref{eq:bary_equi}, respectively, satisfy $P_\# \mu = \bar \mu$, since $\mu$ has support in $\mathcal{R}$.

Note that the projected probability measures on $\mathbb{R}^{N-1}$ are given by
\begin{equation*}
     P_\#\mu_{G_j} \sim \mathcal{N} \left( 0, U^T L_j^\dag U \right) , 
\end{equation*}
where the projected covariance matrices $ U^T L_j^\dag U \in \mathbb{R}^{(N-1) \times (N-1)}$ are strictly positive definite and symmetric.
Thus, by \cite[Theorem~6.1]{agueh2011barycenters} the barycenter problem in the right hand side of \eqref{eq:bary_equi} has a unique solution. 
Moreover, by \citet[Theorem~6.1]{agueh2011barycenters} the unique solution to the barycenter problem in the right hand side of \eqref{eq:bary_equi} is of the form $\bar \mu \sim \mathcal{N}(0, \bar\Sigma)$, and the covariance matrix $\bar \Sigma$ is the unique positive definite solution to \eqref{eq:bary_fb}, where $\Sigma_j = U^T L_j^\dag U$ for $j=1,\dots,m$.
It follows that
\begin{equation*}
    \begin{aligned}
         \sum_{j=1}^m \lambda_j \left( ( U \bar{\Sigma} U^T )^{1/2} L_j^\dag ( U \bar{\Sigma} U^T )^{1/2} \right)^{1/2} & = \sum_{j=1}^m \lambda_j \left(  U \bar{\Sigma}^{1/2} U^T L_j^\dag  U \bar{\Sigma}^{1/2} U^T \right)^{1/2} \\ 
        & = U \left( \sum_{j=1}^m \lambda_j \left( \bar{\Sigma}^{1/2} \Sigma_j \bar{\Sigma}^{1/2}  \right)^{1/2} \right) U^T \\
        & = U \bar{\Sigma} U^T.
    \end{aligned}
\end{equation*}
Thus, the matrix $S= U \bar \Sigma U^T\in \mathbb{R}^{N\times N}$ is a solution to \eqref{eq:bary_eq_degen}.
Moreover, this matrix is symmetric positive semi-definite and has the range $\mathcal{R}$. %, and thus satisfies Assumption~\ref{ass:psd}.
Thus, $\mu \sim \mathcal{N}(0,U \bar \Sigma U^T)$ has support $\mathcal{R}$ and is a minimizer of the left hand side in \eqref{eq:bary_equi}.
%, since it satsfies $P_\# \mu = \bar \mu$.
This constitutes the existence of a positive semi-definite solution to \eqref{eq:bary_eq_degen} with range $\mathcal{R}$ that characterizes the barycenter of $\mu_{G_j}$ for $j=1,\dots,m$.

Now assume that there is another positive semi-definite matrix $\hat S \in \mathbb{R}^{N\times N}$ with $\hat S \neq S$ that solves \eqref{eq:bary_eq_degen} and satisfies $\text{range}(\hat S)=\mathcal{R}$. Then $ U^T \hat S U$ must be a solution to \eqref{eq:bary_fb}, where $\Sigma_j = U^T L_j^\dag U$ for $j=0,\dots,m$, and since this solution is unique it must hold $ U^T \hat S U = \bar \Sigma$.
However, since $\text{range}(\hat S)=\mathcal{R}$ the mapping $\hat S \to U^T \hat S U$ is a bijection on $\mathcal{R}$, and thus it holds $\hat S= U \bar \Sigma U^T$, which contradicts the assumption $\hat S \neq S$. This asserts the uniqueness of the fixed point. 

Finally, we note that since $U \bar \Sigma U^T$ is symmetric positive semi-definite and has the range $\mathcal{R}$, the same holds for the matrix $L = (U \bar \Sigma U^T)^\dag$. Thus, $L$ is in fact a graph Laplacian matrix that satisfies Assumption~\ref{ass:psd}.

\end{proof}

\subsection{Proof of Proposition~\ref{prop:pd_transform} }

In the following we prove Proposition~\ref{prop:pd_transform}, which is used in Algorithm~\ref{alg:bary} for efficient and stable computation.

\begin{proof}
    For ease of notation denote $\mathbf{1} = \mathbf{1}_{N \times N}$ in this proof.
    By \citet[Theorem~4]{gutman2004generalized}, for $j=1,\dots,m$, the matrix $\Sigma_j = L_j^\dag + \frac{1}{N} \mathbf{1}$ has the same eigenvectors as the matrix $ L_j^\dag$, but the zero-eigenvalue is exchanged for an eigenvalue one.
    Thus, the matrices $\Sigma_j$ are symmetric strictly positive-definite.

    Moreover, note that for any matrix $A\in \mathbb{R}^{N\times N}$ satisfying Assumption~\ref{ass:psd} it holds that
    \begin{equation*}
        \left(  A + \frac{1}{N} \mathbf{1} \right)^\dag =   A^\dag + \frac{1}{N} \mathbf{1},\qquad \left(  A + \frac{1}{N} \mathbf{1} \right)^{1/2} =   A^{1/2} + \frac{1}{N} \mathbf{1} . 
    \end{equation*}
    Thus, if $S$ satisfies the matrix equation \eqref{eq:bary_eq_degen} with $L_0,\dots,L_m$, and Assumption~\ref{ass:psd},
    it follows that
    \begin{equation*}
    \begin{aligned}
        \sum_{j=1}^m \lambda_j  \left( \left(S + \frac{1}{N} \mathbf{1}  \right)^{1/2} \left(  L_j + \frac{1}{N} \mathbf{1} \right)^\dag \left( S + \frac{1}{N} \mathbf{1}  \right)^{1/2} \right)^{1/2}
        & = \sum_{j=1}^m \lambda_j  \left( S^{1/2} L_j^\dag S^{1/2} + \frac{1}{N} \mathbf{1}  \right)^{1/2} \\
        & = \sum_{j=1}^m \lambda_j \left( S^{1/2} L_j^\dag S^{1/2} \right)^{1/2} + \frac{1}{N} \mathbf{1} .
    \end{aligned}
    \end{equation*}
    Thus, $\bar{\Sigma} = L^\dag  + \frac{1}{N} \mathbf{1}$ is the unique positive definite solution of the matrix equation \eqref{eq:bary_eq_degen} with inputs $L_j + \frac{1}{N} \mathbf{1}$, for $j=1,\dots,m$.
\end{proof}

\subsection{Proof of Theorem~\ref{thm:geodesic}}

Finally, we present a proof for the analytical expression of the Bures-Wasserstein mean problem for graphs in case only two graphs are given, which can also be interpreted as an interpolation problem.
We note that Theorem~\ref{thm:geodesic} can be proved using Lemma~\ref{lem:ot_proj} and similar arguments as in the proof of Theorem~\ref{thm:bary_unique}.
Here we present a direct and independent proof utilizing similar techniques as the proof of \citet[Proposition~6.1]{xia2014synthesizing}.
\begin{proof}

We use the fact that the optimal transport geodesic between $\mu_{G_0} \sim \mathcal{N}(0, L_0^\dag)$ and $\mu_{G_1} \sim \mathcal{N}(0, L_1^\dag)$ is defined by $\mu_t = \left( (1-t)I +t T \right)_\# \mu_{G_0}$,
where $T$ is a symmetric matrix that satisfies $T L_0^\dag T = L_1^\dag$ \citep{takatsu2010wasserstein}.
It is easy to check that the matrix
\begin{equation} \label{eq:T_interpolation}
T= L_0^{1/2} \left( L_0^{\dag/2} L_1^\dag L_0^{\dag/2} \right)^{1/2} L_0^{1/2}
\end{equation}
satisfies this identity, since all the matrices in the expression \eqref{eq:T_interpolation} have the same range $\mathcal{R} = (\text{span}\{\mathbf{1}_n \})^\perp$.
Thus, the geodesic is of the form $\mu_t \sim \mathcal{N}(0, \Sigma_t)$, and the covariance matrix can be expressed as
\begin{equation*}
\begin{aligned}
\Sigma_t =& \left(  (1-t)I +t T \right) L_0^\dag \left(  (1-t)I +t T \right) \\
=& L_0^{1/2} \left( (1-t) L_0^\dag +t \left( L_0^{\dag/2} L_1^\dag L_0^{\dag/2} \right)^{1/2} \right) L_0^{1/2} L_0^\dag  L_0^{1/2} \left( (1-t) L_0^\dag +t \left( L_0^{\dag/2} L_1^\dag L_0^{\dag/2} \right)^{1/2} \right) L_0^{1/2} \\
=& L_0^{1/2} \left( (1-t) L_0^\dag +t \left( L_0^{\dag/2} L_1^\dag L_0^{\dag/2} \right)^{1/2} \right)^2 L_0^{1/2}
\end{aligned}
\end{equation*}
Since $L_0^\dag$ is symmetric and positive semi-definite, from the first line it follows that $\Sigma_t$ is also symmetric and positive semi-definite. 
Moreover, the range of $\Sigma_t$ is $\mathcal{R}$, as it is composed of matrices with this range.
The pseudo-inverse $\Sigma_t^\dag$ inherits these properties and thus satisfies Assumption~\ref{ass:psd}.

Now assume that there is another mapping $\hat T$ that satisfies $\hat T L_0^\dag T = L_1^\dag$.
However, note that then
$$ \left( L_0^{\dag/2} T L_0^{\dag/2} \right)^2 = L_0^{\dag/2} L_1^\dag L_0^{\dag/2}  = \left( L_0^{\dag/2} \hat T L_0^{\dag/2} \right)^2,  $$
and thus by taking the square root we see that $\hat T$ and $T$ are equal on $\mathcal{R}$.
Hence, $T$ in \eqref{eq:T_interpolation} is the unique mapping with range in $\mathcal{R}$ that satisfies $T L_0^\dag T = L_1^\dag$.
This constitutes the uniqueness of the interpolant $L_t = \Sigma_t^\dag$ satisfying Assumption~\ref{ass:psd}.
\end{proof}

\end{document}